\documentclass{article}

\usepackage{PRIMEarxiv}
\usepackage{amssymb}

\usepackage[utf8]{inputenc} 
\usepackage[T1]{fontenc}    
\usepackage{hyperref}       
\usepackage{url}            
\usepackage{booktabs}       
\usepackage{amsfonts, amsmath, amsthm}       
\usepackage{nicefrac}       
\usepackage{microtype}      
\usepackage{lipsum}
\usepackage{fancyhdr}       
\usepackage{graphicx}       
\usepackage[export]{adjustbox}
\graphicspath{{figure/}}     

\newtheorem{theorem}{Theorem}
\newtheorem{lemma}[theorem]{Lemma}%
\newtheorem{corollary}[theorem]{Corollary}%

\newtheorem{remark}{Remark}%

\newtheorem{definition}{Definition}%

\DeclareMathOperator*{\E}{\mathbb{E}}
\DeclareMathOperator*{\Var}{\mathrm{Var}}
\DeclareMathOperator*{\Cov}{\mathrm{Cov}}
\DeclareMathOperator*{\Corr}{\mathrm{Corr}}

\DeclareFontFamily{U}{matha}{\hyphenchar\font45}
\DeclareFontShape{U}{matha}{m}{n}{
      <5> <6> <7> <8> <9> <10> gen * matha
      <10.95> matha10 <12> <14.4> <17.28> <20.74> <24.88> matha12
      }{}
\DeclareSymbolFont{matha}{U}{matha}{m}{n}

\DeclareMathSymbol{\Lt}{3}{matha}{"CE}
\DeclareMathSymbol{\Gt}{3}{matha}{"CF}

\title{Gradient Descent Fails to Learn High-frequency Functions and Modular Arithmetic 
}

\author{
  Rustem Takhanov, \,\,\,Maxat Tezekbayev, \,\,\,Artur Pak \\
  Department of Mathematics  \\
  Nazarbayev University \\
  Astana, Kazakhstan\\
  \texttt{\{rustem.takhanov, maxat.tezekbayev, artur.pak\}@nu.edu.kz\,\,\,\,\,\,\,\,\,\,\,\,\,\,\,\,} \\
   \And 
  Arman Bolatov\\
  Department of Computer Science  \\
  Nazarbayev University \\
  Astana, Kazakhstan\\  \texttt{arman.bolatov@nu.edu.kz} \\
   \And 
  Zhenisbek Assylbekov \\
  Department of Mathematical Sciences \\
  Purdue University Fort Wayne \\
  Fort Wayne, IN, USA\\
  \texttt{zassylbe@pfw.edu} \\
}

\begin{document}
\maketitle

\begin{abstract}
Classes of target functions containing a large number of approximately orthogonal elements are known to be hard to learn by the Statistical Query algorithms. Recently this classical fact re-emerged in a theory of gradient-based optimization of neural networks. In the novel framework, the hardness of a class is usually quantified by the variance of the gradient with respect to a random choice of a target function.

A set of functions of the form $x\to ax \bmod p$, where $a$ is taken from ${\mathbb Z}_p$, has attracted some attention from deep learning theorists and cryptographers recently.
This class can be understood as a subset of $p$-periodic functions on ${\mathbb Z}$ and is tightly connected with a class of high-frequency periodic functions on the real line.

We present a mathematical analysis of limitations and challenges associated with using gradient-based learning techniques to train a high-frequency periodic function or modular multiplication from examples. We highlight that the variance of the gradient is negligibly small in both cases when either a frequency or the prime base $p$ is large. This in turn prevents such a learning algorithm from being successful.
\end{abstract}

\keywords{modular multiplication\and high-frequency periodic functions\and hardness of learning\and gradient-based optimization\and barren plateau\and statistical query}

\section{Introduction}
Gradient-based optimization is behind the success of modern deep learning methods. Therefore, the study of its limitations is a central theoretical and experimental topic. Already early studies on the subject recognized the vanishing gradient problem as a key reason for the failure to learn high-quality models. Today it is clear that gradient vanishing can be caused by problems with an architecture of a neural network (which inspired many remedies such as LSTM \cite{10.1162/neco.1997.9.8.1735}, residual networks \cite{he2015deep,hardt2017identity}, specific activation functions~\cite{pmlr-v216-takhanov24a}, pre-training~\cite{TAKHANOV2023109777,TAKHANOV2023103819}, weights reusing~\cite{assylbekov-takhanov-2018-reusing} etc.), as well as by an implicit hardness of a class of target functions~\cite{DBLP:conf/stoc/BlumFJKMR94,DBLP:journals/ml/KlivansS07,KLIVANS20092,NIPS2014_3a077244} (or, a combination of reasons).

Recently, a useful relationship was established between gradient-based methods and the so-called Statistical Query (SQ) model \cite{DBLP:conf/stoc/Kearns93,DBLP:conf/soda/FeldmanGV17}. This led to an understanding that classes of target functions with a large statistical query dimension are unlikely to be trainable by modern gradient-based algorithms such as SGD, Adam, Nesterov, etc. In a practical setting, this reveals itself in a vanishing gradient problem, when for many iterations of an algorithm an objective function fluctuates around a certain value (or, easily overfits). This type of behavior is sometimes called the barren plateau phenomenon.

Mathematically, the barren plateau phenomenon can be described in terms of a high concentration of the gradient around its mean (or, equivalently, a negligibly small variance) when a target function is sampled (usually, uniformly) from a hypothesis space. This demonstrates that the information content of the gradient about a key parameter of the target function is, in fact, very small.
A classical example of a hypothesis space for which this variance is provably small is a set of parity functions \cite{DBLP:conf/icml/Shalev-ShwartzS17}. Other examples with obtained upper bounds on the variance of the gradient with respect to a hypothesis set include parameterized quantum circuit training~\cite{McClean2018}, tensor network training~\cite{PhysRevLett.129.270501}, etc~\cite{pmlr-v222-takhanov24a}.
For the current paper, a key example was analyzed by~\cite{DBLP:journals/jmlr/Shamir18}, who considered a set of functions $\{\psi({\mathbf w}^\top {\mathbf x})\mid {\mathbf w}\in {\mathbb R}^d, \|{\mathbf w}\|=r\}$, where $\psi$ is 1-periodic (i.e. $\psi(x) = \psi(x+1)$), both the dimension $d$ and the norm $r$ grow moderately, and proved that the variance is bounded by $\mathcal{O}\big({\rm exp}(-\Theta(d))+{\rm exp}(-\Theta(r))\big)$.

The problem of training basic modular arithmetical operations from data also attracted some attention, though that problem seemingly belongs to a different context. First, it was noticed that the multiplication (or, the division) modulo some prime number $p$ is a mapping that is hard to train even for small primes (e.g. $p=97$)~\cite{gromov2023grokking}. The training process itself has an atypical structure --- first, an SGD algorithm fluctuates around a certain mean value (or, even overfits), and then abruptly falls into a good minimum. The {\em grokking}~\cite{power2022grokking} is a term that was coined for this type of training process. The time until the second regime starts strongly depends on the initial value (whether it is close enough to the so-called ``Goldilocks'' zone) and for larger primes the second regime does not occur at all.  In this line of research, an interest in training modular multiplication is mainly defined by the role the task plays in the demonstration of different aspects of grokking such as, for example, the so-called ``LU mechanism''~\cite{liu2023omnigrok}. Though grokking was reported in more real-world datasets from computer vision, natural language processing and molecule property prediction~\cite{liu2023omnigrok}, training modular multiplication is still viewed as an archetypical case of the phenomenon.

Training the modular arithmetic operations has found applications in neural cryptanalysis. In such applications, one of the multipliers in modular multiplication is sometimes fixed (and called a secret). In~\cite{cryptoeprint:2022/935}, for the needs of lattice cryptography, experiments on training the modular multiplication were extended to include the dot product in ${\mathbb Z}_p^n$. To be specific, the mentioned paper studied the learnability of the mapping $f: {\mathbb Z}_p^n\to {\mathbb Z}_p$, $f(x_1,\cdots, x_n)=a_1x_1+\cdots +a_n x_n \bmod p$ from a set of examples $\{({\mathbf y}_i,f({\mathbf y}_i))\mid {\mathbf y}_i\in {\mathbb Z}_p^n, 1\leq i\leq N\}$ where $(a_1, \cdots, a_n)$ was uniformly sampled from ${\mathbb Z}_p^n$.

In our paper, we are interested in the case $n=1$ only (though our results can be easily generalized to arbitrary $n$). Given some prime number $p>2$, a basic approach of our paper is to view the mapping $x\to ax \bmod p$ as a $p$-periodic function on a ring of integers ${\mathbb Z}$. Periodic functions on ${\mathbb Z}$ are analogous to periodic functions on ${\mathbb R}$, therefore we study these two kinds of target functions together using similar techniques. Thus, we also study the learnability of the function $f: {\mathbb R}\to {\mathbb R}, f(x)=\psi(ax)$, where $\psi$ is 1-periodic and $a$ is sampled uniformly from the set $\{0,\cdots, A-1\}$. In both cases, our proofs are based on the adaptation of the proof technique for orthogonal functions from~\cite{DBLP:conf/icml/Shalev-ShwartzS17}. We emphasize that their result is not directly applicable to the class of functions that we consider in this paper (high-frequency one-dimensional waves or output bits of the modular multiplication) since these functions are not orthogonal with respect to a uniform distribution over the domain. However, they are \emph{approximately} pairwise orthogonal, and the proof of this fact is the core of our work.

In both cases, we prove the existence of the barren plateau phenomenon, when either $p$ is a large prime (e.g. a typical prime base in cryptography), or $A$ is large in the second case. Our paper is organized in the following way. In Section~\ref{overview} we give an overview of obtained results. Section~\ref{sec:experim} is dedicated to empirical verification of our main claims. Due to its technical simplicity, we first give a bound on the variance of the gradient w.r.t. a hypothesis set of periodic functions on the real line (Section~\ref{high-freq}). Then, in Section~\ref{sec:proofs} we give a more elaborate proof of a bound on the variance of the gradient w.r.t. hypothesis sets $\{x\to ax \bmod p\}_{a\in {\mathbb Z}_p^\ast}$ and $\{x\to [ax \bmod p]_r\}_{a\in {\mathbb Z}_p^\ast}$ where $[x]_r$ is the $r$th bit from the end in a binary representation of $x$. Our proof technique is based on bounding the r.h.s. of the Boas-Bellman inequality \cite{boas1941general,bellman1944almost}.
To establish the connection with the SQ model, in Section~\ref{rhsBB} we show that any upper bound on the latter expression directly leads to a lower bound on the so-called SQ dimension of a hypothesis set. Thus, all our bounds can be also turned to lower bounds on the SQ dimension. This implies that high-frequency functions on the real line and modular multiplication are hard not only for gradient-based methods but also for any SQ algorithm.

\subsection{Related Work}
The main source of inspiration for us are works of~\cite{DBLP:journals/jmlr/Shamir18} and \cite{DBLP:conf/icml/Shalev-ShwartzS17}, which, among other things, shows the intractability of learning a class of orthogonal functions using gradient-based methods. We emphasize that their result is not directly applicable to the class of functions that we consider in this paper (high-frequency one-dimensional waves or output bits of the modular multiplication) since these functions are not orthogonal with respect to a uniform distribution over the domain. However, they are \emph{approximately} pairwise orthogonal, and the proof of this fact is the core of our work (Sections~\ref{high-freq} and~\ref{sec:proofs}). In addition, we adapt the proof of intractability of learning parity functions by \cite{DBLP:conf/icml/Shalev-ShwartzS17} to our target function classes by exploiting the Boas-Bellman inequality. This deserves special attention, as it allows us to extend the failure of gradient-based learning to a wider class of approximately orthogonal functions.

It should be noted that the relationship between orthogonal functions and the hardness of learning is not new and has been established in the context of the statistical query (SQ) learning model of \cite{DBLP:conf/stoc/Kearns93}. Moreover, this relationship was characterized by \cite{DBLP:conf/stoc/BlumFJKMR94} in terms of the statistical query dimension of the function class, which essentially corresponds to the largest possible set of functions in the class which are all \emph{approximately} pairwise orthogonal. The hardness of learning a class of boolean functions in an SQ model is usually proven through a lower bound on the statistical query dimension of the class.
In Section~\ref{rhsBB} of the paper we show how any upper bound on the r.h.s. of the Boas-Bellman inequality can be turned into a lower bound on the statistical query dimension. This result builds a bridge between our bounds based on the Boas-Bellman inequality and classical results on the SQ complexity of function classes.

The closest to our result on the hardness of modular multiplication is a result of Yang~\cite{inproceedingsYang} who proved that for a uniform distribution over ${\mathbb Z}_p^n$  and any function 
$\psi: {\mathbb Z}_p\to {\mathbb R}$ where 
${\mathbb E}_{x\sim {\mathbb Z}_p}[\psi(x)]\in [-\frac{1}{\sqrt{2}}, \frac{1}{\sqrt{2}}]$, an SQ algorithm that learns the concept class 
$\mathcal{C} = \{\psi(\langle {\mathbf a}, {\mathbf x} \rangle)\mid {\mathbf a}\in {\mathbb Z}_p^n\}$ significantly better than random guessing requires a running time of $\mathcal{O}(p^{\frac{n-1}{2}})$. This finding suggests that any effective ``booleanization'' of the mapping 
${\mathbf x}\to \langle {\mathbf a}, {\mathbf x} \rangle$ defines a class of functions that is difficult to learn by an SQ algorithm. In fact, \cite{DBLP:conf/soda/FeldmanGV17} demonstrated that a learning algorithm based on a minimization of the expectation of a random convex function falls into the category of SQ algorithms. In conjunction with Yang's result, this implies
that the ``booleanization'' of the mapping 
${\mathbf x}\to \langle {\mathbf a}, {\mathbf x} \rangle$ is hard to learn by any gradient-based convex optimization algorithm. For the case we are mostly interested in (simple modular multiplication), i.e. for $n=1$, Yang's lower bound is trivial. Also, since most contemporary deep learning algorithms focus on optimizing non-convex objectives, this insight does not directly apply to our case.

\subsection{Notation}
Bold-faced lowercase letters ($\mathbf{x}$) denote vectors, bold-faced uppercase letters ($\mathbf{A}$) denote matrices. Regular lowercase letters ($x$) denote scalars (or set elements), and regular uppercase letters ($X$) denote random variables (or random elements). $\|\cdot\|$ denotes the Euclidean norm: $\|\mathbf{x}\|:=\sqrt{\mathbf{x}^\top\mathbf{x}}$. For $\mathbf{x}\in\mathbb{C}^n$, conjugate transpose is denoted by $\mathbf{x}^\dag$. For any finite set or interval $\mathcal{S}$, sampling $X$ uniformly from $\mathcal{S}$ is denoted by $X\sim\mathcal{S}$. 

Given $f:\,\mathbb{N}\to\mathbb{R}$ and $g:\,\mathbb{N}\to\mathbb{R}_+$, we write $f(x)\in \mathcal{O}(g(x))$  if there exist universal constants $\alpha, \beta\in\mathbb{R}_+$ such that for all natural numbers $x>\beta$ we have $|f(x)|\le\alpha g(x)$. If we have $f(x)\in \mathcal{O}(g(x))$ and $g(x)\in \mathcal{O}(f(x))$, then we write $f(x)\in \Theta(g(x))$. When $f,g:U\to\mathbb{R}_+$ are such that $f(x)\leq Cg(x)$ for some universal constant $C>0$, then we write $f\Lt g$.  We write $f\asymp g$  if $f\Lt g$ and $g\Lt f$. If $n\in {\mathbb R}$ is such that $f(x)\in \mathcal{O} (g(x)|\log^n g(x)|)$ then we simply write $f(x)\in\tilde{\mathcal{O}}(g(x))$. 

$\mathbb{Z}_p$ is the set $\{0,1,\ldots, p-1\}$, equipped with two operations, which work as usual addition and multiplication, except that the results are reduced modulo $p$. 
 $\mathbb{Z}_p^\ast$ denotes the set of elements in $\mathbb{Z}_p$ that are relatively prime to $p$. We are mainly interested in the case when $p$ is a prime number greater than 2. In this case $\mathbb{Z}_p^\ast=\{1,\ldots,p-1\}$. By abuse of notation, we sometimes treat elements of $\mathbb{Z}_p$ (and of $\mathbb{Z}_p^\ast$) as integers in $\mathbb{Z}$. Given two positive integers $a$ and $p$, $a \bmod p$ is the remainder of the Euclidean division of $a$ by $p$, where $a$ is the dividend and $p$ is the divisor. For $a,b\in {\mathbb Z}$, we denote ordinary multiplication simply by $ab$ and $ab \bmod p$ by $a\cdot b$ (the base $p$ is always clear from the context). If $a^{-1}$ is an inverse of $a\in {\mathbb Z}_p^\ast$ in the multiplicative group $({\mathbb Z}_p^\ast, \cdot)$, then $b\cdot a^{-1}$ is denoted by $b/a$. Also, below $[\cdot]$ denotes the Iverson bracket, i.e. $[{\rm True}]=1$, $[{\rm False}]=0$. For arbitrary $x$ and $y>0$, let us denote the interval $[x-y,x+y]$ by $x\pm y$.

\section{Main results}\label{overview} 
Let $\psi: {\mathbb R}\to {\mathbb R}$ be a 1-periodic function on the real line that has a bounded variation. The latter condition captures all piecewise Lipschitz functions.
Such a function  induces a family of one-dimensional waves
\begin{equation}
h_a (x) = \psi(ax),
\end{equation}
parametererized by the frequency $a\in \{0,\cdots, A-1\}$ where $A$ is an integer. Let us denote $\mathcal{H}_A = \{h_a\}_{a\in {\mathbb Z}_A}$.

Consider the following supervised learning task. Let us assume that $a$ was generated uniformly from the set $\{0,\cdots, A-1\}$ and our goal is to approximate $h_a(x)$ in terms of $L_2([0,1])$-metrics. That is, we solve the task 
\begin{equation}\label{main-task}
\min\limits_{{\mathbf w}\in \mathcal{O}}\|h_a (x) -\eta({\mathbf w},x)\|^2_{L_2([0,1])}
\end{equation}
where $\eta({\mathbf w},x)$ is some family of functions, parameterized by the weight ${\mathbf w}\in \mathcal{O}\subseteq {\mathbb R}^s$, e.g. given as a neural network. 

For the task~\eqref{main-task} to be solvable by some gradient-based method (e.g., SGD, Nesterov, Adam), the gradient of the objective function at a given point ${\mathbf w}$ should have enough information about the random variable $a$. Following~\cite{DBLP:journals/jmlr/Shamir18}, we measure the information content of the gradient about $a$ by
\begin{equation}
{\rm Var}_{a\sim {\mathbb Z}_A}\big[\nabla_{\mathbf w}\|h_a (x) -\eta({\mathbf w},x)\|^2_{L_2([0,1])}\big],
\end{equation}
which is called the variance of the gradient w.r.t. to the hypothesis space $\mathcal{H}_A$ and is denoted by ${\rm Var}(\mathcal{H}_A, {\mathbf w})$. The smaller ${\rm Var}(\mathcal{H}_A, {\mathbf w})$ is the smaller the information content of the gradient. 

Our key bound reads as 
\begin{equation}\label{main-res}
{\rm Var}(\mathcal{H}_A, {\mathbf w}) \in \tilde{\mathcal{O}}(\frac{1}{\sqrt{A}}),
\end{equation}
where $\tilde{\mathcal{O}}$ hides a factor that depends polynomially on the logarithm of $A$ and factors that are bounded by derivatives of the neural network w.r.t. its weights and the variation of $\psi$. The equation~\eqref{main-res} demonstrates that the gradient becomes non-informative as the highest frequency $A$ grows. Our proof is given in Section~\ref{high-freq} and it is based on a certain combination of the Boas-Bellman inequality with some standard bounds (namely, the Erdős-Turán-Koksma inequality) from the ergodic theory of translations on a two-dimensional torus.

\begin{remark} The learnability by gradient-based methods of the hypothesis space $$\mathcal{H}_r=\{h_{{\mathbf w}}: {\mathbb R}^d\to {\mathbb R}\mid \|{\mathbf w}\|=r, h_{{\mathbf w}}({\mathbf x}) = \psi({\mathbf w}^T{\mathbf x})\}$$ was a subject in~\cite{DBLP:journals/jmlr/Shamir18}, who showed ${\rm Var}(\mathcal{H}_r,w)\in \mathcal{O}({\rm exp}(-\Theta(d))+{\rm exp}(-\Theta(r)))$. This result implies the failure of gradient-based learning when both $d$ and $r$ grow moderately. Our bound shows that the phenomenon of poor learnability of waves maintains in the smallest possible dimension ($d=1$) under the condition that the frequency grows exponentially.
\end{remark}

We proceed to study the observed phenomenon and consider a discrete version of high-frequency waves, now on ${\mathbb Z}$ (not ${\mathbb R}$). Now, let $\psi: {\mathbb Z}\to {\mathbb R}$ be a $p$-periodic function on a ring of integers, i.e. $\psi(x+p)=\psi(x)$, where $p>2$ is prime. Let us introduce the hypothesis
\begin{equation}
\tilde{h}_a: {\mathbb Z}\to {\mathbb R}, \tilde{h}_a (x) = \psi(ax)
\end{equation}
parametererized by the frequency $a\in {\mathbb Z}^\ast_p$. Let us denote $\tilde{\mathcal{H}}_p = \{\tilde{h}_a\}_{a\in {\mathbb Z}^\ast_p}$. Note that any function from $\tilde{\mathcal{H}}_p$ is defined by its restriction on the set ${\mathbb Z}_p$, therefore, we alternatively can study the hypothesis set $\mathring{\mathcal{H}}_p = \{\mathring{h}_a:{\mathbb Z}_p\to {\mathbb R}\mid \mathring{h}_a(x) = t(a\cdot x),a\in {\mathbb Z}_p^\ast \}$ where $t$ is a restriction of $\psi$ on ${\mathbb Z}_p$.

Our interest towards the hypothesis set $\mathring{\mathcal{H}}_p$ is motivated by specific cases when (a) $t(x)=\frac{(a\cdot x)-\frac{p}{2}}{\sqrt{\frac{p^2}{12}-\frac{p}{6}}}[x\ne 0]$ or (b) $t(x)=[x]_r$ where $[x]_r$ is the $r$th bit from the end in a binary representation of $x$. The case (a) corresponds to the standardized hypothesis set $\{x\to a\cdot x\}_{a\in {\mathbb Z}_p^\ast}$ and the case (b) to the set $\{x\to [a\cdot x]_r\}_{a\in {\mathbb Z}_p^\ast}$. Let us define the variance w.r.t. $\mathring{\mathcal{H}}_p$ by
\begin{equation}\label{var-zp}
{\rm Var}(\mathring{\mathcal{H}}_p, {\mathbf w}) = {\rm Var}_{a\sim {\mathbb Z}^\ast_p}\big[\nabla_{\mathbf w} {\mathbb E}_{x\sim {\mathbb Z}^\ast_p}[L(\mathring{h}_a (x), \eta({\mathbf w},x))]\big],
\end{equation}
where $L$ is a square loss function for the case (a) and $L(y,y')=l(yy')$, where $
l$ is some 1-Lipschitz function, for the case (b). 

In Sections~\ref{sec:proofs} and~\ref{applications} we prove that in both cases we have
\begin{equation}\label{discrete}
{\rm Var}(\mathring{\mathcal{H}}_p, {\mathbf w}) \in \tilde{\mathcal{O}}(\frac{1}{\sqrt{p}}),
\end{equation}
where, again, $\tilde{\mathcal{O}}$ hides factors that depend polynomially on the logarithm of $p$ and on derivatives of our model $\eta({\mathbf w},x)$ w.r.t. its weights (Theorems~\ref{standard} and~\ref{r-var}).

\begin{remark}
From the first site, a proof of the bound~\eqref{discrete} is substantially different from the proof of~\eqref{main-res}. In fact, there is an analogy between the two. Both proofs are based on the application of the Boas-Bellman inequality. When we prove~\eqref{discrete}, we first obtain a general bound on ${\rm Var}(\mathring{\mathcal{H}}_p, {\mathbf w})$ in terms of properties of the Discrete Fourier Transform of $t$, and then apply the general bound to the cases (a) and (b) of modular multiplication.  As was already mentioned, a proof of~\eqref{main-res} is based on the Erdős-Turán-Koksma inequality, i.e. on the statement whose canonical proof is itself based on Fourier analysis. We believe that both bounds are edges of a more general theory (which, probably, can be formulated for periodic functions on general abelian groups), but this falls out of the scope of our paper.
\end{remark}

\begin{remark}\label{rem:sgd_fails} From our bounds on the variances, using Theorem 10 of~\cite{DBLP:journals/jmlr/Shamir18}, one can derive the following statement: an output of any algorithm with access to an approximate gradient oracle can distinguish between different values of the parameter $a$ of the target function $h_a$ (or, of $\mathring{h}_a$) only if the number of its iterations is substantially greater than $\tilde{\Theta}(A^{1/6})$ (or, $\tilde{\Theta}(p^{1/6})$). This statement implies that for large $A$ (or $p$), a gradient-based method fails for the task of training the set of targets $\mathcal{H}_A$ (or $\mathring{\mathcal{H}}_p$) if the hypothesis $h_a$ (or, $\mathring{h}_a$) depends on the parameter $a$ in a highly sensitive way (which is usually the case).
In the last section of the paper we also obtain lower bounds on the statistical query dimension of our classes, and they grow at least as $\tilde{\Theta}(A^{1/3})$ (or, $\tilde{\Theta}(p^{1/3})$).
\end{remark}
\section{Empirical Verification}\label{sec:experim}
\paragraph{Failure to learn high-frequency waves on ${\mathbb R}$.} 
We verified the hardness of the task~\eqref{main-task} for the stochastic gradient descent algorithm directly by minimizing the MSE loss that approximates the square of the $L_2([0,1])$-distance.
We chose $\psi(x)=\{x\}$ where $\{x\}$ is the fractional part of $x\in {\mathbb R}$. 

The training data consisted of pairs of values \((x, \{ax\})\), where \(x\) is uniformly sampled from $[-100, 100]$. The training dataset contained \(1000\) pairs of \((x, \{ax\})\) for each selected \(a\). We trained a 3-layer neural network with dimensions of layers $[1,64,128, 1]$ and the ReLU activation function. For each $A$, the training was repeated 5 times, and each time \(a\) was picked uniformly at random from $\mathbb{Z}_A$. We used Adam optimizer with a learning rate of $0.001$, batch size $1000$, MSE loss, and we trained for 100 epochs. The results for different values of $A$ are shown in Figure~\ref{fig:high_freq_verif}. For each $A$, an average MSE loss (across 5 runs) as a function of an epoch is depicted.

Notice that for a large $a\in {\mathbb N}$ and a random $X$ uniformly distributed on $[-100,100]$, the random variable $h_a(X)=\psi(aX)$ is distributed approximately uniformly on $[0,1]$. If we consider the mean of the uniform distribution on $[0,1]$ as a baseline approximation of $h_a(x)$, then the corresponding squared $L_2([0,1])$-distance of such an approximation equals the variance of this distribution, i.e. $\frac{1}{12}$. As Figure~\ref{fig:high_freq_verif} shows, our model  is unable to achieve the MSE loss smaller than the trivial $\frac{1}{12}$.

\begin{figure}
    \centering
    \includegraphics[width=.5\textwidth]{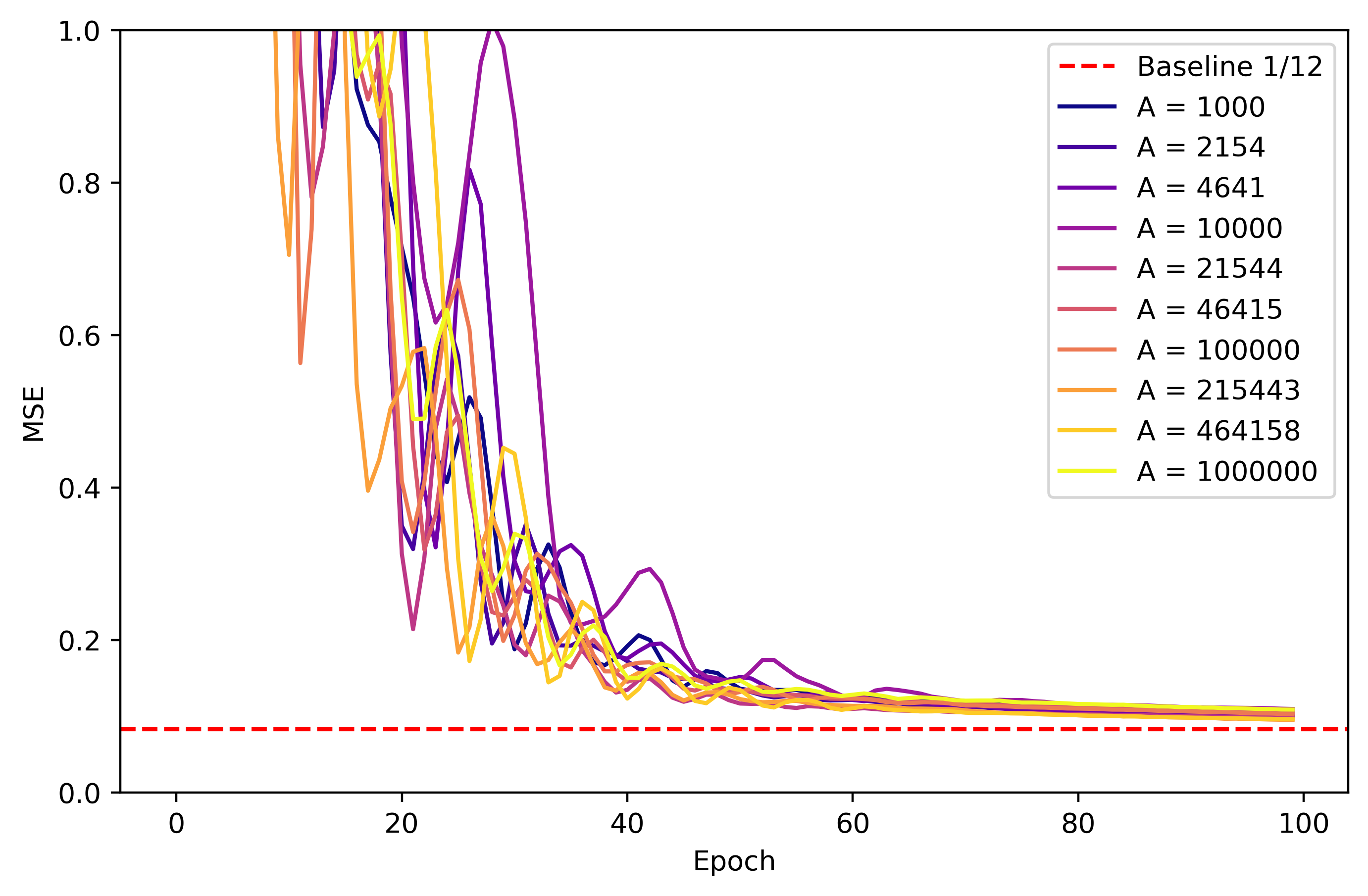}
    \caption{Learning the high-frequency wave on ${\mathbb R}$ with a 3-layer dense network with ReLU activation. For each $A$, the coefficient $a$ was sampled randomly 5 times from ${\mathbb Z}_A$, and an average MSE loss as a function of an epoch is depicted. The horizontal asymptote corresponds to ${\rm MSE}=\frac{1}{12}$.}
    \label{fig:high_freq_verif}    
\end{figure}

\paragraph{Concentration of the Gradient.} Let us verify empirically the statement~\eqref{discrete} for the last bit of modular multiplication. Let $\eta(\mathbf{w}, x)$ be a neural network\footnote{3-layer dense neural network with 1000 neurons on each hidden layer, sigmoid activation, binary cross-entropy loss.} that we may want to train to learn the mapping $x \mapsto (-1)^{a\cdot x}$, $x\in\mathbb{Z}_p^\ast$, 
where $\mathbf{w}\in\mathbb{R}^s$ are all parameters of the neural network. Since $(-1)^{a\cdot x}=(-1)^{[a\cdot x]_1}$, this mapping computes the last bit of modular multiplication $a\cdot x$. 
Let $\nabla L_a(\mathbf{w})$ be the gradient of the binary cross-entropy loss function at $\mathbf{w}$. We sample $\mathbf{w}_1,\ldots,\mathbf{w}_{20}$ from $\mathbb{R}^s$ using the default PyTorch initializer\footnote{For each dense layer of shape $d_\text{in}\times d_\text{out}$, PyTorch initializes its parameters uniformly at random from the interval $\left[-1/{\sqrt{d_\text{in}}},1/{\sqrt{d_\text{in}}}\right]$, where $d_\text{in}$ is the size of the input, and $d_\text{out}$ is the size of the output.}, and for each $\mathbf{w}_i$, we compute 
\begin{align}
    v(\mathbf{w}_i)&=\E_{A\sim\mathbb{Z}_p^\ast}\left\|\nabla L_A(\mathbf{w}_i)-\E_{A'\sim\mathbb{Z}_p^\ast}\nabla L_{A'}(\mathbf{w}_i)\right\|^2,\\
    g(\mathbf{w}_i)&=\E_{X\sim\mathbb{Z}_p^\ast}\left\|\frac{\partial}{\partial\mathbf{w}}\eta(\mathbf{w}, X)\right\|^2.\label{eq:g_w}
\end{align}
According to Theorem~\ref{r-var}, the values $\frac{v(\mathbf{w}_i)}{g(\mathbf{w}_i)}$ should be of order $\tilde{\mathcal{O}}\left(\frac{1}{\sqrt{p}}\right)$. Thus, we plot $\E_{i\sim\{1,\ldots,20\}}\left[\frac{v(\mathbf{w}_i)}{g(\mathbf{w}_i)}\cdot\sqrt{p}\right]$ against $p$ in Figure~\ref{fig:thm1_verif}.
\begin{figure}
\begin{minipage}[t]{.48\textwidth}
    \centering
    \includegraphics[width=.9\textwidth]{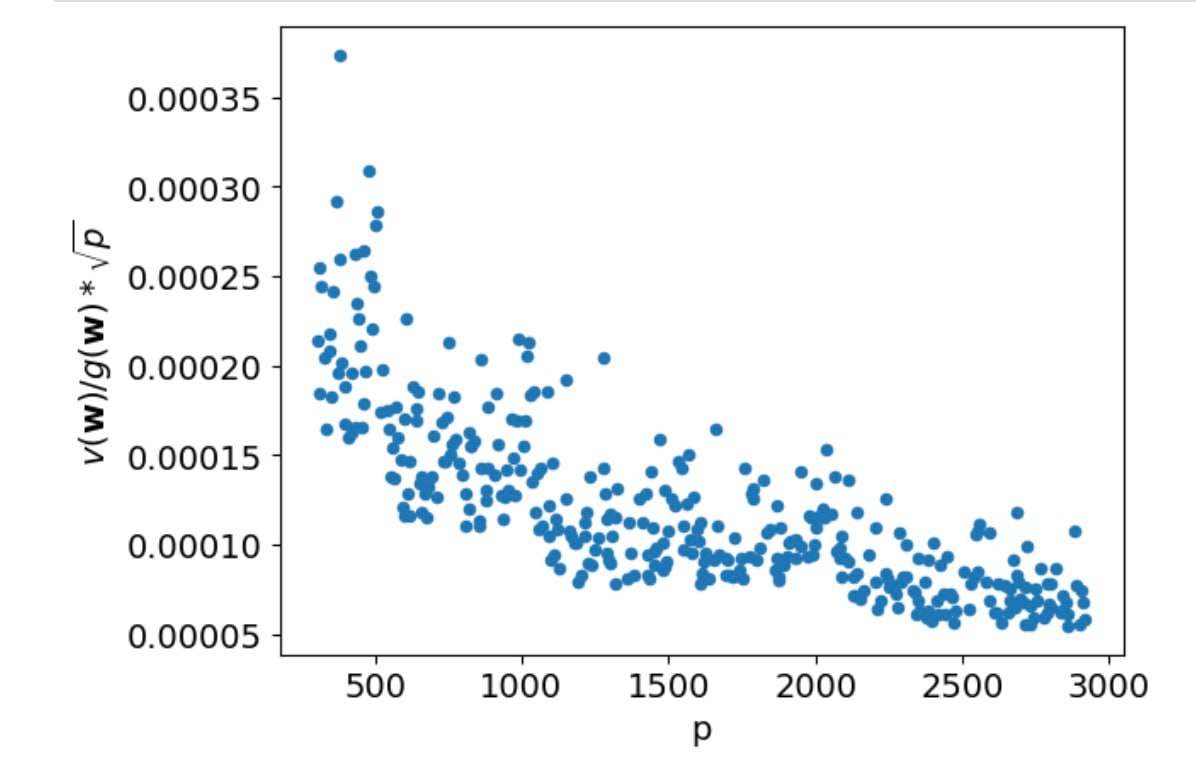}
    \caption{Verifying the statement of Theorem~\ref{r-var}. For prime numbers $p$ in $[300,3000]$, we plot the left-hand side of \eqref{discrete} divided by the average squared norm of the neural network's gradient \eqref{eq:g_w} and multiplied by $\sqrt{p}$. The resulting curve is of order $\tilde{\mathcal{O}}(1)$. Moreover, it even decreases.}
    \label{fig:thm1_verif}    
\end{minipage}\hfill\begin{minipage}[t]{.48\textwidth}
    \centering
    \includegraphics[width=.9\textwidth]{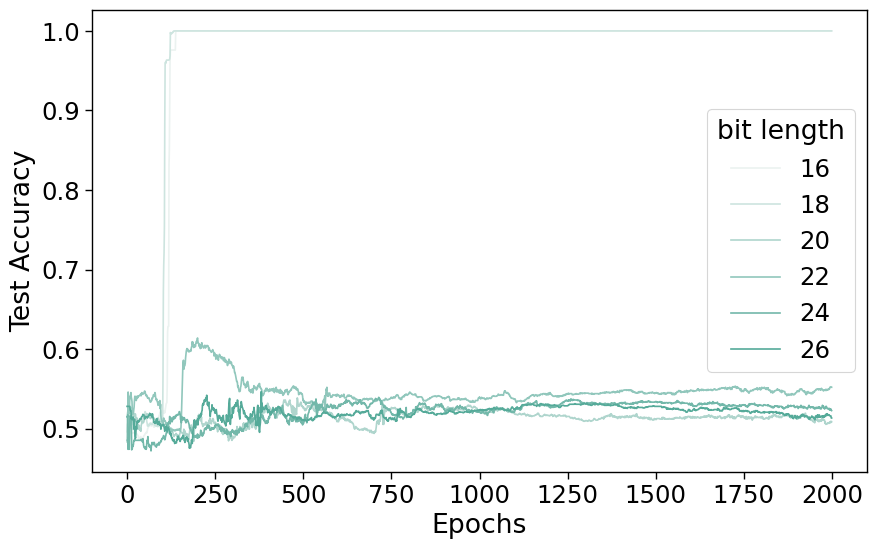}
    \caption{Learning the parity bit of multiplication modulo $p$ with a 3-layer width-1000 dense network. Darker shades correspond to longer bit lengths. For each bitlength $n$, $p$ is chosen randomly from the prime numbers in the interval $[2^{n-1}, 2^n-1]$.}
    \label{fig:sgd_failure}    
\end{minipage}
\end{figure}
As we can see, this expression is bounded as $p$ grows, which confirms the statement of the theorem. In fact, Figure~\ref{fig:one_over_p} suggests that our upperbound \eqref{discrete} can be improved, i.e. the variance of the gradient decays like $\tilde{\mathcal{O}}(\frac{1}{p})$, not like $\tilde{\mathcal{O}}(\frac{1}{\sqrt{p}})$ (at a random point ${\mathbf w}$). 

\begin{figure}
    \centering
    \includegraphics[width=.6\textwidth]{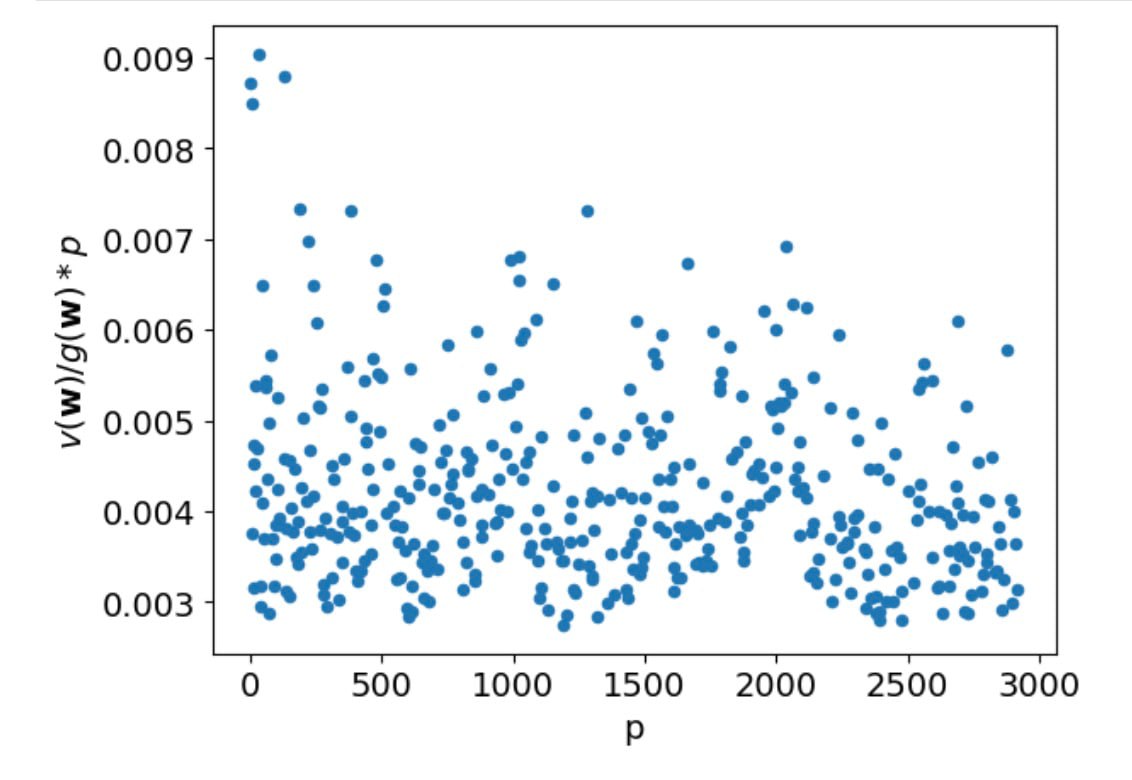}
 \caption{$\E_{i\sim\{1,\ldots,20\}}\left[\frac{v(\mathbf{w}_i)}{g(\mathbf{w}_i)}\cdot p\right]$ against $p$ for modular multiplication.}
 \label{fig:one_over_p}   
\end{figure}
\paragraph{Failure to learn modular multiplication.} According to Remark~\ref{rem:sgd_fails}, any gradient-based method most likely will fail to learn the parity bit of modular multiplication. To test this claim, we generated a labeled sample 
$$
(x_1, (-1)^{a\cdot x_1}), \ldots, (x_m, (-1)^{a\cdot x_m})
$$
where $x_1, \ldots, x_m$, and $a$ are taken randomly from $\mathbb{Z}_p^\ast$. Using this sample, we trained a dense 3-layer neural network with 1000 neurons in each hidden layer. We used Adam with a learning rate of 0.001 (default), $m=5000$, a 70/30 split between training and test sets, batch size 100, and we trained for 2000 epochs. The results for different bitlengths $n$ are shown in Figure~\ref{fig:sgd_failure}. The prime base $p$ for each bitlength $n$ was taken randomly from the prime numbers in the interval $[2^{n-1}, 2^n-1]$. We can see that as the bit length increases, the chances of successful learning decrease, as predicted by our theory.

\subsection{Additional Experiments}
Here we present the results of experiments on the learnability of modular multiplication \emph{itself} and of \emph{all} its bits. 
\paragraph{Low correlation of multiplication by different numbers.} We computed the mean squared  covariance 
\begin{equation}
    \E_{A,B\sim\mathbb{Z}_p^\ast}\left(\Cov_{X\sim\mathbb{Z}_p^\ast}[A\cdot X, B\cdot X]\right)^2\label{eq:mscov}
\end{equation}
for prime numbers in the interval $[3, 500]$. The results are shown in Figure~\ref{fig:cov}.
\begin{figure}
    \centering
    \includegraphics[width=.5\textwidth]{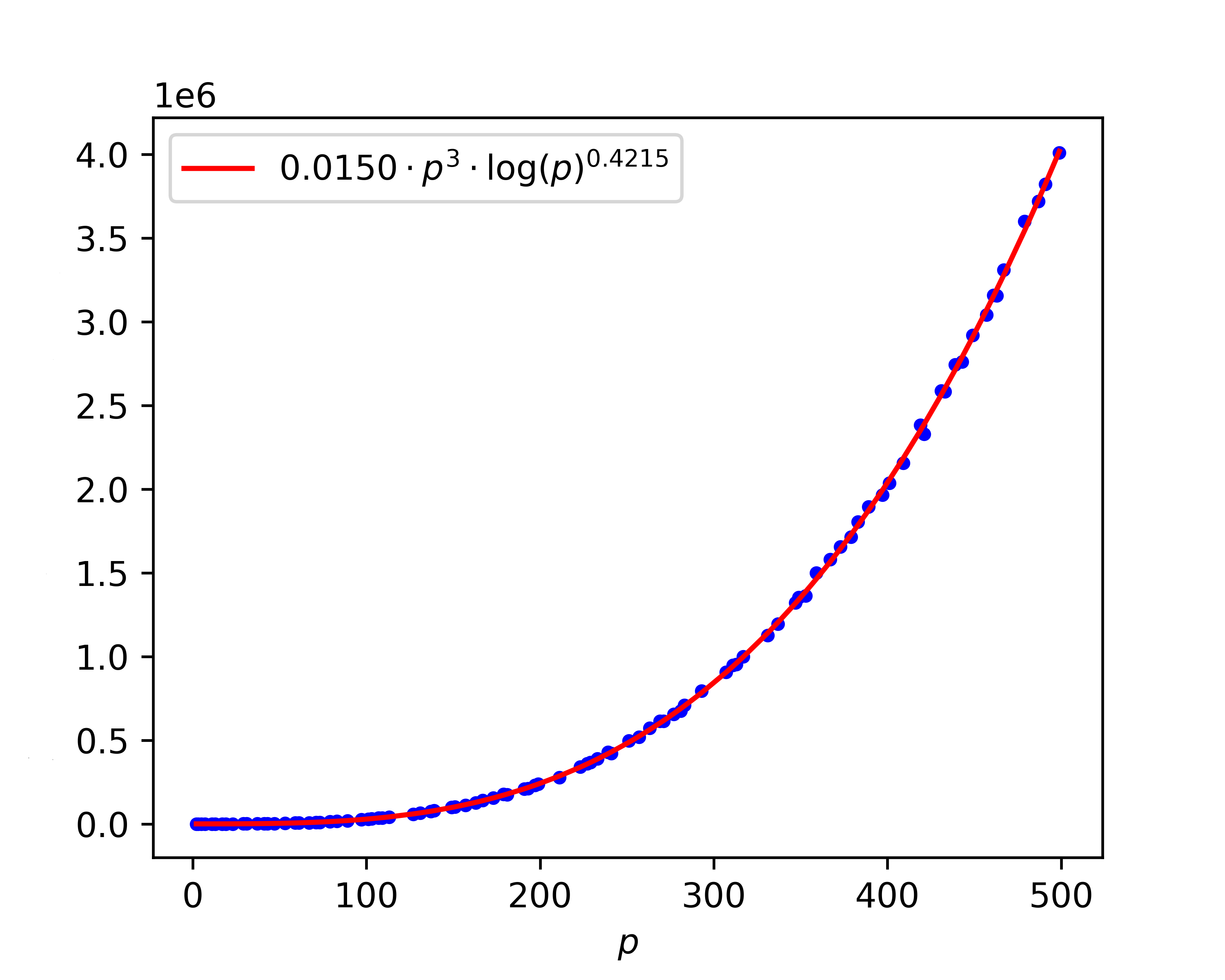}
    \caption{Mean squared covariance between two multiplications, $a\cdot X$ and $b\cdot X$, when $X$ is a random variable uniformly distributed on $\mathbb{Z}_p^\ast$.}
    \label{fig:cov}
\end{figure}
As we can see, the expression \eqref{eq:mscov} fits the curve $p\mapsto 0.015\cdot p^3\cdot(\log p)^{0.42}$ well. 
In Corollary~\ref{add-exper} of subsection~\ref{whole-out} we prove a slightly weaker bound $ \E_{A,B\sim\mathbb{Z}_p^\ast}\left(\Cov_{X\sim\mathbb{Z}_p^\ast}[A\cdot X, B\cdot X]\right)^2\Lt p^3 \log^2 p$. 
 This suggests that $\Cov_{X\sim\mathbb{Z}_p^\ast}[a\cdot X,b\cdot X]\in \tilde{\mathcal{O}}(p^{3/2})$ on average over $a,b\in\mathbb{Z}_p^\ast$. Since the variance of the multiplication by $a$ modulo $p$ is
$$
\Var_{X\sim\mathbb{Z}_p^\ast}[a\cdot X]=\frac{1}{p-1}\sum_{k=1}^{p-1}k^2-\left(\frac{1}{p-1}\sum_{k=1}^{p-1}k\right)^2=\frac{p^2}{12}-\frac{p}{6}\in \mathcal{O}(p^2)
$$
we have that the average correlation is
\begin{equation}
\Corr_{X\sim\mathbb{Z}_p^\ast}[a\cdot X,b\cdot X]\in \tilde{\mathcal{O}}\left(\frac{1}{\sqrt{p}}\right).\label{eq:cor_bound}
\end{equation}
Thus, using this estimate in the Boas-Bellman inequality for the class of ``standardized'' modular multiplication functions $\{f_a(x)\mid a\in\mathbb{Z}_p^\ast\}$, where 
$$
f_a(x)=\frac{a\cdot x-\frac{p}2}{\sqrt{\frac{p^2}{12}-\frac{p}{6}}},\qquad x\in\mathbb{Z}_p^\ast,
$$
one can show that this class is also hard to learn by gradient-based methods. A rigorous proof of the bound \eqref{eq:cor_bound}, together with the concentration of the gradient (Theorem~\ref{standard}), can be found in subsection~\ref{whole-out}. 

\paragraph{Failure to learn all bits of modular multiplication.} Here we follow the experimental setup for the target $x\to [a\cdot x]_1$ with the difference that the output of the neural network is not only the parity bit, but all the bits of modular multiplication. As a loss function, we use the sum of the cross-entropies for each bit. The results for two different bit lengths are shown in Figure~\ref{fig:all_bits}.
\begin{figure}
    \centering
    \,\qquad\includegraphics[width=.45\textwidth, valign=t]{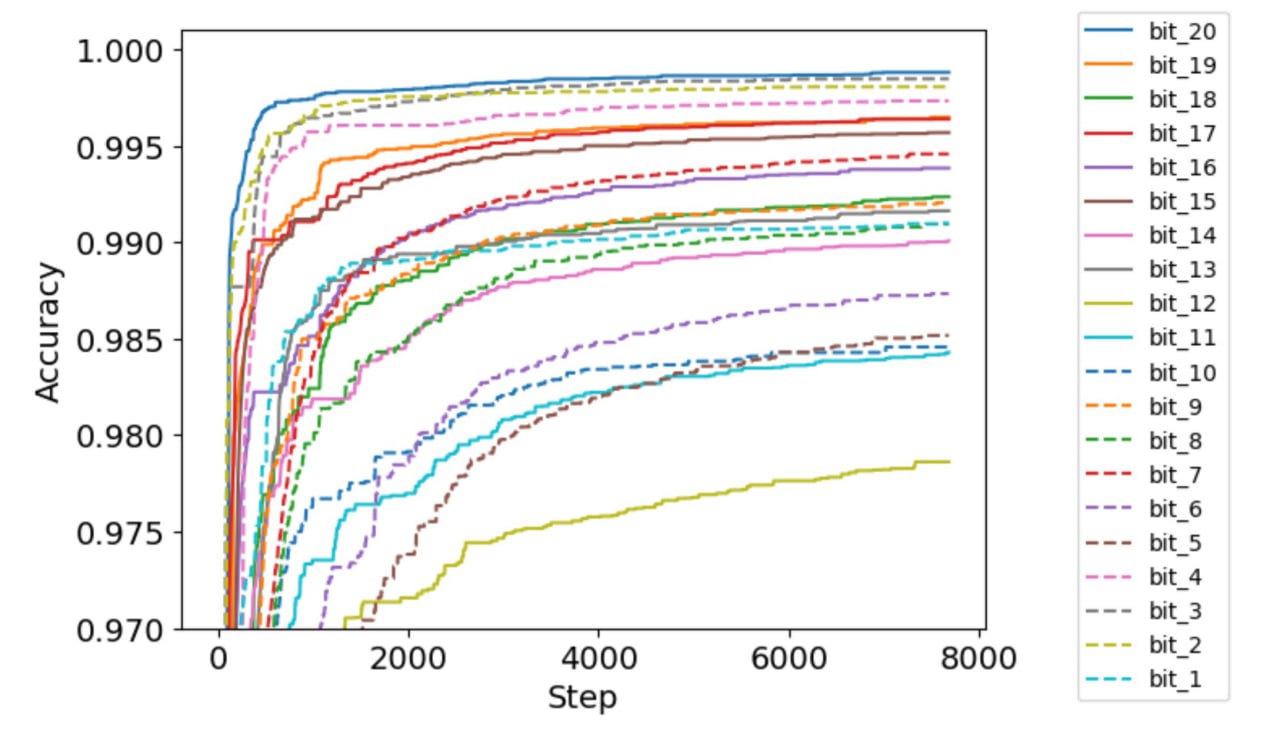}\hfill\includegraphics[width=.4\textwidth, valign=t]{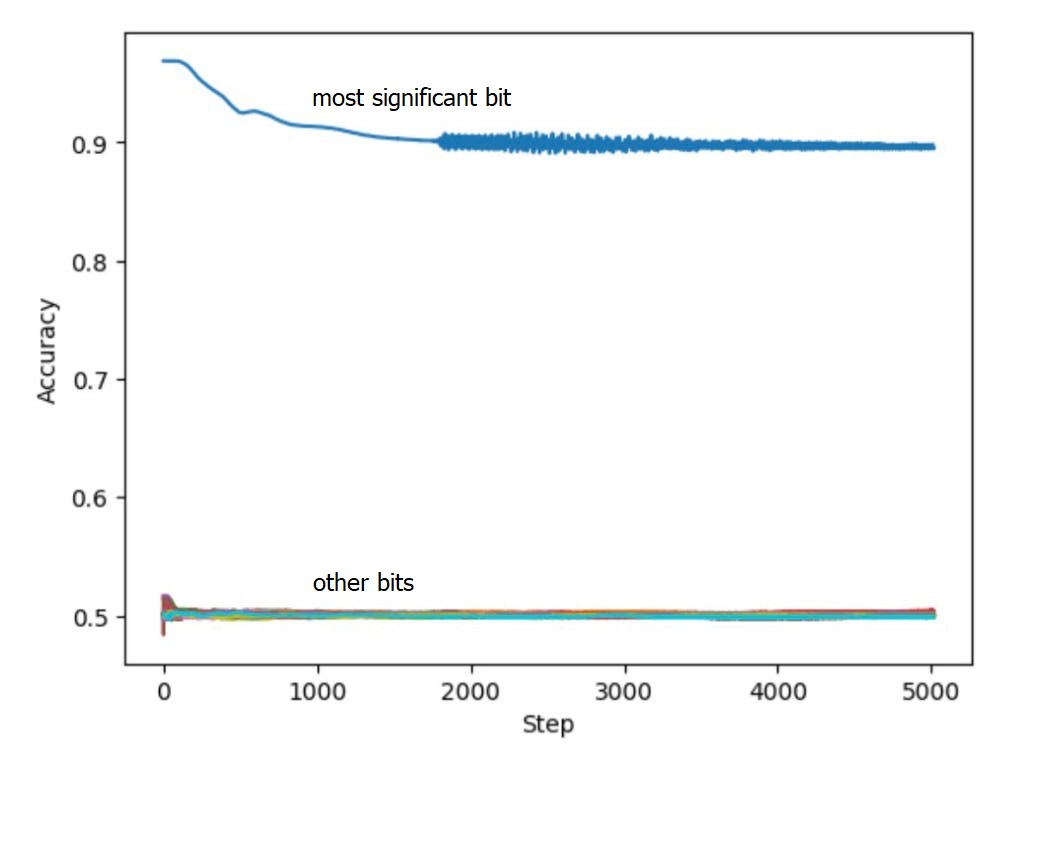}\qquad\,
    \caption{Test Accuracies when learning \emph{all} bits of multiplication modulo $p$ with a single neural network. Bitlengths of $p$: 20 (left) and 40 (right).  Accuracies are higher for larger $r$ because the classes are unbalanced.}
    \label{fig:all_bits}
\end{figure}
As in the case of one bit, we see that for a longer bit length, the gradient method is not able to learn all the bits of modular multiplication.

\section{High-frequency waves on the real line}\label{high-freq}
Let $L_2([0,1])$ be the space square-integrable functions on $[0,1]$ with the inner product $\langle f,g\rangle = \int_0^1 f(x)g(x)dx$. This inner product induces the norm $\|f\|_{L_2([0,1])} = \sqrt{\int_0^1 f(x)^2dx}$. For a function with a bounded variation $f: {\mathbb R}\to {\mathbb R}$, we denote $$\|f\|_{BV}\triangleq\sup_{\substack{N\in {\mathbb N} \\ 0= x_0\leq ...\leq x_N= 1}}\sum_{i=1}^N |f(x_i)-f(x_{i-1})|.$$ 
Let $\psi: {\mathbb R}\to {\mathbb R}$ be a periodic function with period 1, i.e. $\psi(x+1)=\psi(x)$, and mean 0, i.e. $\int_0^1\psi(x)dx=0$. Also, we assume that $\psi$ has a bounded variation, i.e. $\|\psi\|_{BV}<+\infty$.

For $A\in {\mathbb N}$, let us introduce the hypothesis space
\begin{equation*}
\begin{split}
\mathcal{H}_A = \{h_a: {\mathbb R}\to {\mathbb R}\mid h_a(x) = \psi(ax), a\in  {\mathbb Z}_A\}.
\end{split}
\end{equation*}
Let $\eta:\boldsymbol{\Omega}\times {\mathbb R}\to {\mathbb R}$ be a function, where $\boldsymbol{\Omega}$ is an open subset of ${\mathbb R}^s$. We make only the most general assumptions on the form of $\eta$ such as continuity and the existence of a partial derivative w.r.t. the first variable in almost all points, i.e. $\nabla_{\mathbf w} \eta({\mathbf w}, x)$.
The variance of the gradient w.r.t. $\mathcal{H}_A$ is defined as
\begin{equation*}
\begin{split}
{\rm Var}(\mathcal{H}_A, {\mathbf w}) \triangleq {\rm Var}_{a\sim {\mathbb Z}_A}\big[\nabla_{\mathbf w}\| \eta({\mathbf w}, x)-h_a(x)\|^2_{L_2([0,1])}\big], {\mathbf w}\in \boldsymbol{\Omega}.
\end{split}
\end{equation*}
We define that the variance of a random vector is a sum of the variances of components.
Let us denote $\frac{\partial \eta({\mathbf w}, x)}{\partial w_i}\in L_2([0,1])$ by $g_i({\mathbf w})$ and set $g({\mathbf w})=[g_1({\mathbf w}), \cdots, g_s({\mathbf w})]^\top$, i.e. $g$ is the gradient. Note that $g({\mathbf w})\in L^s_2([0,1])$ and we define $\|g({\mathbf w})\|^2_{L^s_2([0,1])}\triangleq \sum_{i=1}^s \|g_i({\mathbf w})\|^2_{L_2([0,1])}$. The following theorem shows that for a large $A$, the variance of the gradient w.r.t. $\mathcal{H}_A$ is very small. 
\begin{theorem}\label{ergodic} There exist  universal constants $C, A_0>0$ such that
\begin{equation*}
\begin{split}
{\rm Var}(\mathcal{H}_A, {\mathbf w})\leq C\|g({\mathbf w})\|^2_{L^s_2([0,1])} \|\psi\|^2_{BV} A^{-1/2}(\log A+1)^{5/2},
\end{split}
\end{equation*}
for any $A>A_0$.
\end{theorem}
In the proof of the latter theorem we will need the following classical fact, which is a generalization of Bessel's inequality.
\begin{lemma}[Boas-Bellman inequality] \label{lem:boas_bellman} Let $h_1,\ldots,h_{m},g$ be elements of some Hilbert space $\mathcal{H}$ equipped with the inner product $\langle\cdot,\cdot\rangle_{\mathcal{H}}$. Then
$$
\sum_{i=1}^{m}\langle h_i, g\rangle^2_{\mathcal{H}}\le\|g\|^2_{\mathcal{H}}\left(\max_i\|h_i\|^2_{\mathcal{H}}+\sqrt{\sum_{i\ne j}\langle h_i,h_j\rangle^2_{\mathcal{H}}}\right).
$$
\end{lemma}

\begin{proof}[Proof of Theorem~\ref{ergodic}] The variance of the gradient w.r.t. $\mathcal{H}_A$ can be bounded in the following way:
\begin{equation*}
\begin{split}
{\rm Var}_{a\sim {\mathbb Z}_A}\big[\partial_{w_i} {\mathbb E}_{x\sim [0,1]}[(\eta({\mathbf w},x)-h_a(x))^2]\big]
={\rm Var}_{a\sim {\mathbb Z}_A}\big[ {\mathbb E}_{x\sim [0,1]}[2(\eta({\mathbf w}, x)-h_a(x))\partial_{w_i} \eta({\mathbf w},x)]\big]\\
=4{\rm Var}_{a\sim {\mathbb Z}_A}\big[ {\mathbb E}_{x\sim [0,1]} [h_a(x)\partial_{w_i} \eta({\mathbf w}, x)]\big]\leq
\frac{4}{A}\sum_{a\in {\mathbb Z}_A}\langle h_a, g_i({\mathbf w})\rangle^2.
\end{split}
\end{equation*}
Above we used that ${\rm Var}_{a\sim {\mathbb Z}_A}[l(a)]={\rm Var}_{a\sim {\mathbb Z}_A}[l(a)+C]$ for any $C$ that does not depend on $a$. Also, we used that the variance does not exceed the second moment. 
The Boas-Bellmann inequality gives us
\begin{equation}\label{boas-bellman}
\begin{split}
\sum_{a=0}^{A-1}\langle h_a, g_i({\mathbf w})\rangle^2\leq \|g_i({\mathbf w})\|^2_{L_2([0,1])}\big(\max_{a\in {\mathbb Z}_A} \|h_a\|^2+\sqrt{\sum_{a\ne b\in {\mathbb Z}_A}\langle h_a,h_b\rangle^2}\big).
\end{split}
\end{equation}
Recall that $\langle h_a,h_b\rangle = \int_0^1 \psi(ax)\psi(bx)dx$. 
Therefore,
\begin{equation}\label{pairs-sum}
\begin{split}
\sum_{a, b\in {\mathbb Z}_A}\langle h_a,h_b\rangle^2 = \sum_{a, b\in {\mathbb Z}_A} \int_{[0,1]^2}\psi(ax)\psi(bx)\psi(ay)\psi(by)dxdy
= \int_{[0,1]^2} Q(x,y)^2 dxdy = \|Q\|^2_{L_2([0,1]^2)}.
\end{split}
\end{equation}
where $Q(x,y) \triangleq \sum_{a\in {\mathbb Z}_A}\psi(ax)\psi(ay)$. This function can be represented as
$$
Q(x,y) =\sum_{a=0}^{A-1} \psi^2(T^a_{(x,y)}(0,0)),
$$
where $C={\mathbb R}/{\mathbb Z}$ is a factor-group (topologically isomorphic to a circle), $\psi^2: C\times C\to {\mathbb R}$, $\psi^2(x', y') = \psi(x')\psi(y')$ and $T_{(x,y)}:C\times C\to C\times C$, $T_{(x,y)}(x',y')=(x'+x \bmod 1,y'+y\bmod 1)$, $T^a_{(x,y)}=T_{(x,y)}\circ\cdots\circ T_{(x,y)}$ ($a$ times). 
Note that $C\times C$ is a two-dimensional torus and $T_{(x,y)}$ is a translation operator on that torus. The standard Lebegue measure on $C\times C$ is preserved by $T_{(x,y)}$. Recall that an operator $O: M\to M$ preserving a measure $\mu$ on $M$ is called ergodic if almost surely over a choice of $x\sim \mu$ we have $\lim_{N\to +\infty}\frac{1}{N}\sum_{i=0}^{N-1}f(T^{i}(x)) = \int_{M}fd\mu$ for any $f$ integrable w.r.t. $\mu$.
The Weyl-von Neumann theorem~\cite{Sinai1977} states that the translation operator $T_{(x,y)}$ is ergodic on the 2-dimensional torus $C\times C$ if $ax+by=c,a,b,c\in {\mathbb Z}$ implies $a=b=c=0$. Thus, for a random choice of $x,y,z,t\in [0,1]$, almost surely the following is satisfied:
\begin{equation*}
\begin{split}
\lim_{A\to +\infty}\frac{1}{A}\sum_{a=0}^{A-1} \psi^2(T^a_{(x,y)}(z,t)) = \int_{[0,1]^2} \psi(x')\psi(y')dx'dy'=0.
\end{split}
\end{equation*}

Now to estimate the speed of convergence for $(z,t)=(0,0)$ we need the notion of the discrepancy. Let ${\mathbf x}_1, \cdots, {\mathbf x}_N\in [0,1]^s$, then
\begin{equation*}
\begin{split}
D^\ast_N({\mathbf x}_1, \cdots, {\mathbf x}_N) \triangleq \sup_{0\leq u_1,\cdots, u_s\leq 1}|\frac{|\{i\in [N]\mid {\mathbf x}_i\in [0,u_1)\times \cdots \times [0,u_s)\}|}{N}-\prod_{j=1}^s u_j| .
\end{split}
\end{equation*}
is called the discrepancy of a set $\{{\mathbf x}_1, \cdots, {\mathbf x}_N\}$.
Also, we will need the following two classical results from the theory of low-discrepancy sequences (for details, see pages 116 and 151 of~\cite{kuipers2012uniform}). 
\begin{lemma}[Koksma-Hlawka inequality~\cite{kuipers2012uniform}] 
Let $f:[0,1]^2\to {\mathbb R}$ be such that $\|f(x,1)\|_{BV}<\infty$, $\|f(x,0)\|_{BV}<\infty$, $\|f(1,x)\|_{BV}<\infty$, $\|f(0,x)\|_{BV}<\infty$ and
\begin{equation*}
\begin{split}
V^2(f)\triangleq
\sup_{\substack{N,M\in {\mathbb N} \\ 0=x_0\leq \cdots \leq x_N=1 \\0=y_0\leq \cdots \leq y_M=1}} \sum_{i=1}^{N}\sum_{j=1}^{M}|f(x_{i},y_{i})-f(x_{i-1},y_{i})-f(x_{i},y_{i-1})+f(x_{i-1},y_{i-1})| <\infty.
\end{split}
\end{equation*}
Then, for any ${\mathbf x}_1, \cdots, {\mathbf x}_N\in [0,1]^2$, ${\mathbf x}_i = [x_{i1}, x_{i2}]$, we have
\begin{equation*}
\begin{split}
|\frac{1}{N}\sum_{i=1}^N f({\mathbf x}_i) - \int_{[0,1]^2}f({\mathbf x})d{\mathbf x}|\leq \|f(x,1)\|_{BV}D^\ast_N(x_{11}, \cdots, x_{N1})+\\
\|f(1,x)\|_{BV}D^\ast_N(x_{12}, \cdots, x_{N2})+
V^2(f)D^\ast_N({\mathbf x}_1, \cdots, {\mathbf x}_N).
\end{split}
\end{equation*}
\end{lemma}
Our plan is to apply the latter lemma to the case of $f(x,y) = \psi^2(x,y) = \psi(x)\psi(y)$ and ${\mathbf x}_i = T^{i-1}_{(x,y)}(0,0)$, $i\in [N]$. For that case we have $\max(\|f(x,1)\|_{BV}, \|f(1,x)\|_{BV})\leq \|\psi\|^2_{BV}$ and $V^2(f)\leq \|\psi\|^2_{BV}$.  By construction $D^\ast_N(x_{1c}, \cdots, x_{Nc})\leq D^\ast_N({\mathbf x}_1, \cdots, {\mathbf x}_N)$,  $c\in \{0,1\}$, therefore 
\begin{equation*}
\begin{split}
|\frac{1}{A}\sum_{a=0}^{A-1} \psi^2(T^a_{(x,y)}(0,0)) |\Lt 
\|\psi\|^2_{BV}  D^\ast_N({\mathbf x}_1, \cdots, {\mathbf x}_N).
\end{split}
\end{equation*}
Thus, it is enough to bound $D^\ast_N({\mathbf x}_1, \cdots, {\mathbf x}_N)$.
This can be done using the Erd{\"o}s-Turán-Koksma inequality.

\begin{lemma}[Erd{\"o}s-Turán-Koksma inequality~\cite{kuipers2012uniform}] \label{lem:erdos} Let ${\mathbf x}_1, \cdots, {\mathbf x}_N\in [0,1]^s$ and $H\in {\mathbb N}$. Then,
\begin{equation*}
\begin{split}
D^\ast_N({\mathbf x}_1, \cdots, {\mathbf x}_N) \leq C_s \Big(\frac{1}{H}+\sum_{{\mathbf h}: 0<\|{\mathbf h}\|_\infty\leq H}\frac{1}{r({\mathbf h})}\big|\frac{1}{N}\sum_{l=1}^N e^{2\pi {\rm i}  {\mathbf h}^\top {\mathbf x}_l}\big|\Big),
\end{split}
\end{equation*}
where $r({\mathbf h}) = \prod_{j=1}^s\max(|h_j|,1)$ for ${\mathbf h} = [h_1, \cdots, h_s]\in {\mathbb Z}^s$ and  $C_s$ is a constant that only depends on $s$.
\end{lemma}
The Erd{\"o}s-Turán-Koksma inequality  helps us to bound the discrepancy, and therefore 
\begin{equation*}
\begin{split}
|\frac{1}{A}\sum_{a=0}^{A-1} \psi^2(T^a_{(x,y)}(0,0)) | 
\Lt 
\|\psi\|^2_{BV}\Bigg( 
\frac{1}{H}+\sum_{a,b\in [-H,H]\cap {\mathbb Z}, (a,b)\ne (0,0)} \frac{\frac{1}{A}|\sum_{n=0}^{A-1} e^{2\pi {\rm i} (anx+bny)}|}{\max(1,|a|)\max(1,|b|)}\Bigg),
\end{split}
\end{equation*}
for any $H\in {\mathbb N}$.

Note that  $\sum_{n=0}^{A-1} e^{2\pi {\rm i} (anx+bny)} = \frac{e^{2\pi {\rm i} A(ax+by)}-1}{e^{2\pi {\rm i} (ax+by)}-1}=\frac{|\sin(\pi A(ax+by))|}{|\sin(\pi (ax+by))|}$. Using the triangle inequality we conclude
\begin{equation}\label{q-squared}
\begin{split}
 \|Q\|_{L_2([0,1]^2)}\Lt 
A \|\psi\|^2_{BV} \Bigg(\frac{1}{H}+\\
+\Big\|\sum_{a,b\in [-H,H]\cap {\mathbb Z}, (a,b)\ne (0,0)} \frac{|\sin(\pi A(ax+by))|}{A|\sin(\pi (ax+by))|\max(1,|a|)\max(1,|b|)}\Big\|_{L_2([0,1]^2)}\Bigg).
\end{split}
\end{equation}
Now our plan will be to estimate the latter expressions and obtain an upper bound on $\|Q\|_{L_2([0,1]^2)}$. Further, we will combine it with the equality~\eqref{pairs-sum} and bound the expression $\sum_{a, b\in {\mathbb Z}_A}\langle h_a,h_b\rangle^2$. The last expression bounds the needed variance according to Boas-Bellman inequality~\eqref{boas-bellman}.

The following lemma is a key fact behind our bound of the RHS of~\eqref{q-squared}.
\begin{lemma}\label{sin-over-sin} For $\omega\geq 1, r\in {\mathbb N}$, we have $\int_{-1}^1\frac{|\sin(r\omega x)|}{|\sin(\omega x)|}dx\Lt 1+ \log r$.
\end{lemma}
\begin{proof}
 Recall that $x\pm y = [x-y,x+y]$.  First note that if $\omega x\notin \Omega$ where $\Omega=\bigcup_{k\in {\mathbb Z}}(\pi k\pm \frac{\pi}{4})$, then $\frac{|\sin(r\omega x)|}{|\sin(\omega x)|}\leq \frac{1}{|\sin(\omega x)|}\leq \sqrt{2}$ and $\int_{[-1,1]\setminus\Omega}\frac{|\sin(r\omega x)|}{|\sin(\omega x)|}dx\leq 2\sqrt{2}$. 

Let us now estimate $\int_{[-1,1]\cap\Omega}\frac{|\sin(r\omega x)|}{|\sin(\omega x)|}dx$. If $\omega x\in \pi k\pm \frac{\pi}{4}$, then $2|\sin(\omega x)|\geq |\omega x-\pi k|$. We have 
\begin{equation*}
\begin{split}
\int_{(\pi k\pm \frac{\pi}{4})/\omega}\frac{|\sin(r\omega x)|}{|\sin(\omega x)|}dx\leq 2\int_{(\pi k\pm \frac{\pi}{4})/\omega}\frac{|\sin(r(\omega x-\pi k))|}{|\omega x-\pi k|}dx\\
= \frac{2}{\omega}\int_{-\pi/4}^{\pi/4}\frac{|\sin(rx)|}{|x|}dx=\frac{2}{\omega}\int_{-r\pi/4}^{r\pi/4}\frac{|\sin(x)|}{|x|}dx\Lt \frac{\log r +1}{\omega}.
\end{split}
\end{equation*}
Since the number $|k|$ can be at most $\lfloor\frac{\omega+\frac{\pi}{4}}{\pi}\rfloor$, the total integral is asymptotically bounded by $1+\frac{\log r +1}{\omega}\frac{\omega+\frac{\pi}{4}}{\pi}\Lt 1+\log r$. 
\end{proof}
\begin{lemma}\label{one} For $(a,b)\in {\mathbb Z}^2\setminus \{(0,0)\}$, we have ${\mathbb E}_{(x,y)\sim [0,1]^2}[\frac{|\sin(\pi A(ax+by))|}{|\sin(\pi (ax+by))|}]\Lt 1+\log A$.
\end{lemma}
\begin{proof} The following chain of inequalities can be checked:
\begin{equation*}
\begin{split}
{\mathbb E}_{(x,y)\sim [0,1]^2}\Big[\frac{|\sin(\pi A(ax+by))|}{|\sin(\pi (ax+by))|}\Big]\leq \int_{x^2+y^2\leq 2} \frac{|\sin(\pi A(ax+by))|}{|\sin(\pi (ax+by))|}dxdy\\=
\frac{1}{a^2+b^2}\int_{x^2+y^2\leq 2(a^2+b^2)} \frac{|\sin(\pi Ax)|}{|\sin(\pi x)|}dxdy=
\frac{1}{a^2+b^2}\int_{-\sqrt{2(a^2+b^2)}}^{\sqrt{2(a^2+b^2)}} \frac{|\sin(\pi Ax)|\sqrt{2(a^2+b^2)-x^2}}{|\sin(\pi x)|}dx\\ \leq 
\frac{\sqrt{2}}{\sqrt{a^2+b^2}}\int_{-\sqrt{2(a^2+b^2)}}^{\sqrt{2(a^2+b^2)}} \frac{|\sin(\pi Ax)|}{|\sin(\pi x)|}dx=2\int_{-1}^{1} \frac{|\sin(\pi A\sqrt{2a^2+2b^2}x)|}{|\sin(\pi \sqrt{2a^2+2b^2} x)|}dx.
\end{split}
\end{equation*}
Using Lemma~\ref{sin-over-sin} we conclude that the latter expression is asymptotically bounded by $1+\log A$.
\end{proof}
\begin{lemma}\label{two} For $(a,b),(c,d)\in {\mathbb Z}^2\setminus \{(0,0)\}$, we have ${\mathbb E}_{(x,y)\sim [0,1]^2}\big[\frac{|\sin(\pi A(ax+by)) \sin(\pi A(cx+dy))|}{|\sin(\pi (ax+by))\sin(\pi (cx+dy))|}\big]\Lt A(1+\log A)$.
\end{lemma}
\begin{proof} Using $\frac{|\sin(\pi A(cx+dy))|}{|\sin(\pi (cx+dy))|}\leq A$ and the previous lemma, we obtain
\begin{equation*}
\begin{split}
{\mathbb E}_{(x,y)\sim [0,1]^2}\Big[\frac{|\sin(\pi A(ax+by)) \sin(\pi A(cx+dy))|}{|\sin(\pi (ax+by))\sin(\pi (cx+dy))|}\Big] \leq 
A {\mathbb E}_{(x,y)\sim [0,1]^2}\Big[\frac{|\sin(\pi A(ax+by))|}{|\sin(\pi (ax+by))|}\Big]\Lt A(1+\log A).
\end{split}
\end{equation*}
\end{proof}
Now we are ready to bound  $$\Big\|\sum_{a,b\in [-H,H]\cap {\mathbb Z}, (a,b)\ne (0,0)} \frac{|\sin(\pi A(ax+by))|}{A|\sin(\pi (ax+by))|\max(1,|a|)\max(1,|b|)}\Big\|_{L_2([0,1]^2)}^2$$  using Lemmas~\ref{one} and~\ref{two} by:
\begin{equation*}
\begin{split}
{\mathbb E}_{(x,y)\sim [0,1]^2}\Big(
\hspace{-25pt}\sum_{\scriptscriptstyle a,b\in [-H,H]\cap {\mathbb Z}, (a,b)\ne (0,0)}\hspace{-10pt} \frac{|\sin(\pi A(ax+by))|}{A|\sin(\pi (ax+by))|\max(1,|a|)\max(1,|b|)}\Big)^2 
\end{split}
\end{equation*}
\begin{equation*}
\begin{split}
= A^{-2}\sum_{\substack{\scriptscriptstyle a,b\in [-H,H]\cap {\mathbb Z}, (a,b)\ne (0,0) \\ \scriptscriptstyle c,d\in [-H,H]\cap {\mathbb Z}, (c,d)\ne (0,0)}} \frac{1}{\max(1,|a|)\max(1,|b|)\max(1,|c|)\max(1,|d|)}\\
\times {\mathbb E}_{(x,y)\sim [0,1]^2}\Big[\frac{|\sin(\pi A(ax+by)) \sin(\pi A(cx+dy))|}{|\sin(\pi (ax+by))\sin(\pi (cx+dy))|}\Big] 
\end{split}
\end{equation*}
\begin{equation*}
\begin{split}
\Lt A^{-1}(\log A+1)(\sum_{a=-H}^H \frac{1}{\max(1,|a|)})^4\Lt 
A^{-1}(\log A+1)(\log H +1)^4.
\end{split}
\end{equation*}
After we set $H\asymp A^{1/2}(\log A+1)^{-5/2}$, using~\eqref{q-squared} we obtain 
\begin{equation}\label{BBbound1}
\|Q\|_{L_2([0,1]^2)}\in A\|\psi\|^2_{BV} \mathcal{O} \big(  A^{-1/2}(\log A+1)^{5/2}\big).
\end{equation}
Finally, using~\eqref{boas-bellman}, we obtain
\begin{equation*}
\begin{split}
\sum_{a=0}^{A-1}\langle h_a, g_i({\mathbf w})\rangle^2\leq \|g_i({\mathbf w})\|^2_{L_2([0,1])}\big(\max_{a\in {\mathbb Z}_A} \|h_a\|^2+\sqrt{\|Q\|^2_{L_2([0,1]^2)}}\big)\\ \leq 
\|g_i({\mathbf w})\|^2_{L_2([0,1])}\Big(\|\psi\|^2_{BV}+\|\psi\|^2_{BV} \mathcal{O}\big(A^{1/2}(\log A+1)^{5/2}\big)\Big) \\ \in 
\|g_i({\mathbf w})\|^2_{L_2([0,1])}\|\psi\|^2_{BV} \mathcal{O}\big(A^{1/2}(\log A+1)^{5/2}\big).
\end{split}
\end{equation*}
The latter directly leads to the statement of Theorem,
\begin{equation*}
\begin{split}
{\rm Var}_{a\sim {\mathbb Z}_A}\big[\nabla_{\mathbf w} {\mathbb E}_{x}(\eta({\mathbf w}, x)(x)-h_a(x))^2\big]\\ \in 
\|g({\mathbf w})\|^2_{L^s_2([0,1])} \|\psi\|^2_{BV} A^{-1} \mathcal{O}\big(A^{1/2}(\log A+1)^{5/2}\big) \in 
 \|g({\mathbf w})\|^2_{L^s_2([0,1])} \|\psi\|^2_{BV} \mathcal{O}\big(A^{-1/2}(\log A+1)^{5/2}\big).
\end{split}
\end{equation*}
\end{proof}

\section{The case of $p$-periodic function on ${\mathbb Z}$}\label{sec:proofs}
Let $t: {\mathbb Z}^\ast_p\to {\mathbb R}$ be some function. Let us extend this function to ${\mathbb Z}_p$ by defining $t(0)$ arbitrarily. 
Let $a\in {\mathbb Z}_p^\ast$. Let us define a function $\mathring{h}_a: {\mathbb Z}_p\to {\mathbb R}$ by
    \begin{equation}\label{eq:DL_parity_bit}
    \mathring{h}_a(x)\triangleq t(a\cdot x),
    \end{equation}
and introduce a hypothesis set
    \begin{equation}\label{eq:DL_parity_bit}
    \mathring{\mathcal{H}}_p = \{\mathring{h}_a(x)\mid a\in {\mathbb Z}_p^\ast\}.
    \end{equation}
Throughout the section we assume that elements of ${\mathbb R}^p$ (or, ${\mathbb C}^p$) are functions from ${\mathbb Z}_p$ to ${\mathbb R}$ (or, ${\mathbb C}$) and introduce an inner product in ${\mathbb R}^p$ (or, ${\mathbb C}^p$) by 
$$
\langle f,g\rangle \triangleq \sum_{x\in {\mathbb Z}_p}f(x)^\dag g(x).
$$
We will be interested in the variance of the gradient w.r.t.  $\mathring{\mathcal{H}}_p$ defined by
\begin{equation}
{\rm Var}(\mathring{\mathcal{H}}_p, {\mathbf w}) \triangleq {\rm Var}_{A\sim {\mathbb Z}^\ast_p}\big[\nabla_{\mathbf w} {\mathbb E}_{x\sim {\mathbb Z}^\ast_p}[L(\mathring{h}_A (x), \eta({\mathbf w},x))]\big],
\end{equation}
where $L$ is either a square loss function (for regression tasks) or $L(y,y')=l(yy')$, where $
l$ is some 1-Lipschitz function (for binary classification tasks).

Again, let $\eta({\mathbf w},x)$ be differentiable w.r.t. ${\mathbf w}$ in almost all points and let us denote $\partial_{w_i} \eta({\mathbf w},x)\in {\mathbb R}^{p}$ by $g_i({\mathbf w})$. A natural norm of the gradient is given by $\|g_i\|^2_{\ast} \triangleq {\mathbb E}_{X\sim {\mathbb Z}_p} \big[(\partial_{w_i} \eta({\mathbf w},X))^2\big]$. Also, let $g({\mathbf w})\triangleq [g_1({\mathbf w}), \cdots, g_s({\mathbf w})]^\top$ and $\|g\|^2_{\ast s}\triangleq \sum_{i=1}^s \|g_i\|_{\ast}^2$. 

\begin{theorem}\label{BBBB2} For the square loss function, we have
\begin{equation}
\begin{split}
{\rm Var}_{a\sim {\mathbb Z}^\ast_p}\big[\nabla_{\mathbf w} {\mathbb E}_{x\sim {\mathbb Z}^\ast_p}[L(\mathring{h}_a (x), \eta({\mathbf w},x))]\big] \Lt \|g({\mathbf w})\|^2_{\ast s}\sqrt{{\mathbb E}_{Y\sim\mathbb{Z}_p^\ast}[f(Y)^2]},
\end{split}
\end{equation}
 where 
\begin{equation}\label{eq:f}
f(Y)\triangleq \frac{1}{p-1}\sum_{x\in {\mathbb Z}_p^\ast} t(x)t(Y\cdot x).
\end{equation}
\end{theorem}
\begin{proof} Since the l.h.s. does not depend on the value of $t$ at 0, we can set $t(0)=0$. 
In the case of the square loss, an application of the Boas-Bellman inequality to the family $\{\mathring{h}_a(x)\}_{a\in {\mathbb Z}_p^\ast}$ gives us 
\begin{equation}\label{BBBB}
\begin{split}
{\rm Var}_{A\sim {\mathbb Z}^\ast_p}\big[\partial_{w_i} {\mathbb E}_{x\sim {\mathbb Z}^\ast_p}[(\mathring{h}_A (x)- \eta({\mathbf w},x))^2]\big]  \leq  
\frac{4}{(p-1)^2}{\mathbb E}_{A\sim {\mathbb Z}_p^\ast}\big[\langle t(A\cdot x), g_i(w, x)\rangle^2\big] = \\
\frac{4}{(p-1)^3}\sum_{a\in {\mathbb Z}_p^\ast}\langle t(a\cdot x), g_i(w, x)\rangle^2  \leq 
\frac{4\|g_i({\mathbf w})\|^2}{(p-1)^3}\big(\max_{a\in {\mathbb Z}_p^\ast} \|\mathring{h}_a\|^2+\sqrt{\sum_{a\ne b}\langle \mathring{h}_a,\mathring{h}_b\rangle^2}\big).
\end{split}
\end{equation}
Due to $\max_{a\in {\mathbb Z}_p^\ast} \|\mathring{h}_a\|^2\leq \sqrt{\sum_{a\ne b}\langle \mathring{h}_a,\mathring{h}_b\rangle^2}$, the r.h.s. is bounded by a factor of $\sqrt{\sum_{a, b\in {\mathbb Z}_p^\ast}\langle \mathring{h}_a,\mathring{h}_b\rangle^2}$. As the following lemma shows, it itself can be bounded by a factor of $\sqrt{{\mathbb E}_{Y\sim\mathbb{Z}_p^\ast}[f(Y)^2]}$.

\begin{lemma} \label{lem:sum_of_squares}    
We have,
\begin{equation*}
\begin{split}
\sum_{a\in {\mathbb Z}_p^\ast}\sum_{b\in {\mathbb Z}_p^\ast}  \langle \mathring{h}_a, \mathring{h}_b\rangle^2 = (p-1)^4{\mathbb E}_{Y\sim\mathbb{Z}_p^\ast}[f(Y)^2].
\end{split}
\end{equation*}
\end{lemma}
\begin{proof}[Proof of Lemma] The following sequence of identities can be directly checked:

\begin{equation*}
\begin{split}
\frac{1}{(p-1)^2}\sum_{a\in {\mathbb Z}_p^\ast}\sum_{b\in {\mathbb Z}_p^\ast}{\mathbb E}_{X\sim {\mathbb Z}_p^\ast}[\mathring{h}_a(X) \mathring{h}_b(X)]^2
=\E_{A,B\sim {\mathbb Z}_p^\ast}\left[\E_{X\sim {\mathbb Z}_p^\ast}[\mathring{h}_A(X) \mathring{h}_B(X)]^2\right] \\
=\E_{A,B\sim {\mathbb Z}_p^\ast}\left[\E_{X\sim {\mathbb Z}_p^\ast}\left[t( A X)  t( B X)\right]^2\right]
=\E_{A,B\sim {\mathbb Z}_p^\ast}\left[\E_{X\sim {\mathbb Z}_p^\ast}\left[t(A X) t((B\cdot A^{-1})\cdot A X)\right]^2\right]\\
\end{split}
\end{equation*}
\begin{equation*}
\begin{split}
=\E_{A,B\sim {\mathbb Z}_p^\ast}\left[\E_{X\sim  {\mathbb Z}_p^\ast}\left[t(X) t((B\cdot A^{-1})\cdot X)\right]^2\right]
=\E_{Y\sim {\mathbb Z}_p^\ast}\left[\E_{X\sim \mathbb{Z}_p^\ast}
\left[t(X) t(Y\cdot X)\right]^2\right] =
{\mathbb E}_{Y\sim\mathbb{Z}_p^\ast}[f(Y)^2].
\end{split}
\end{equation*}
Since ${\mathbb E}_{X\sim {\mathbb Z}_p^\ast}[\mathring{h}_a(X) \mathring{h}_b(X)]=\frac{\langle \mathring{h}_a, \mathring{h}_b\rangle}{p-1}$, we conclude $\sum_{a,b\in {\mathbb Z}_p^\ast}  \langle \mathring{h}_a, \mathring{h}_b\rangle^2 = (p-1)^4{\mathbb E}_{Y\sim\mathbb{Z}_p^\ast}[f(Y)^2]$. 
Lemma proved.
\end{proof}

Thus, we have
\begin{equation*}
\begin{split}
{\rm Var}_{a\sim {\mathbb Z}^\ast_p}\big[\partial_{w_i} {\mathbb E}_{x\sim {\mathbb Z}^\ast_p}[(\mathring{h}_a (x)- \eta({\mathbf w},x))^2]\big] \Lt \|g_i({\mathbf w})\|^2_{\ast}\sqrt{{\mathbb E}_{Y\sim\mathbb{Z}_p^\ast}[f(Y)^2]},
\end{split}
\end{equation*}
from which the statement of the theorem for a square loss function directly follows. 
\end{proof}

Using ideas from Appendix B.1 of~\cite{DBLP:conf/icml/Shalev-ShwartzS17}, the case of a binary classification task and a 1-Lipschitz loss can be treated analogously. Let us give a proof of the following theorem for completeness.

\begin{theorem}\label{BBBB21} Suppose that $t(x)=t_1(x)+c, x\in {\mathbb Z}^\ast_p$ and a restriction of $t$ to ${\mathbb Z}^\ast_p$ is $\{-1,1\}$-valued. For $L(y,y')=l(yy')$, where $l$ is some 1-Lipschitz function, we have
\begin{equation}
\begin{split}
{\rm Var}_{a\sim {\mathbb Z}^\ast_p}\big[\nabla_{\mathbf w} {\mathbb E}_{x\sim {\mathbb Z}^\ast_p}[L(\mathring{h}_a (x), \eta({\mathbf w},x))]\big] \Lt \|g({\mathbf w})\|^2_{\ast s}\sqrt{{\mathbb E}_{Y\sim\mathbb{Z}_p^\ast}[f_1(Y)^2]},
\end{split}
\end{equation}
where $f_1(y) = {\mathbb E}_{X\sim {\mathbb Z}_p^\ast}t_1(X)t_1(y\cdot X)$.
\end{theorem}
\begin{proof} W.l.o.g. $t(0)=t_1(0)=0$.
We have
\begin{equation*}
\begin{split}
{\rm Var}_{a\sim {\mathbb Z}^\ast_p}\big[\partial_{w_i} {\mathbb E}_{x\sim {\mathbb Z}^\ast_p}[l(t (a\cdot x)\eta({\mathbf w},x))]\big]  = 
{\rm Var}_{a\sim {\mathbb Z}^\ast_p}\big[ {\mathbb E}_{x\sim {\mathbb Z}^\ast_p}[\partial_{w_i}\eta({\mathbf w},x) l'(t (a\cdot x)\eta({\mathbf w},x))]\big].
\end{split}
\end{equation*}
Since $t$ is $\{-1,1\}$-values, we have $l'(t (a\cdot x)\eta({\mathbf w},x)) = \frac{l'(\eta({\mathbf w},x))+l'(-\eta({\mathbf w},x))}{2}+
\frac{l'(\eta({\mathbf w},x))-l'(-\eta({\mathbf w},x))}{2}t (a\cdot x)$. Therefore, the last variance equals
\begin{equation*}
\begin{split}
{\rm Var}_{a\sim {\mathbb Z}^\ast_p}\big[ {\mathbb E}_{x\sim {\mathbb Z}^\ast_p}[\partial_{w_i}\eta({\mathbf w},x) (\frac{l'(\eta({\mathbf w},x))+l'(-\eta({\mathbf w},x))}{2}+
\frac{l'(\eta({\mathbf w},x))-l'(-\eta({\mathbf w},x))}{2}t (a\cdot x))]\big] \\ = 
{\rm Var}_{a\sim {\mathbb Z}^\ast_p}\big[ {\mathbb E}_{x\sim {\mathbb Z}^\ast_p}[\partial_{w_i}\eta({\mathbf w},x)
\frac{l'(\eta({\mathbf w},x))-l'(-\eta({\mathbf w},x))}{2}t_1 (a\cdot x)]\big]
\end{split}
\end{equation*}
Above we used that ${\rm Var}_{a\sim {\mathbb Z}^\ast_p}(f(a)+c)={\rm Var}_{a\sim {\mathbb Z}^\ast_p}(f(a))$ is $c$ does not depend on $a$. The last variance is bounded by the second moment, i.e. by
\begin{equation*}
\begin{split}
{\mathbb E}_{a\sim {\mathbb Z}^\ast_p}\big[ {\mathbb E}_{x\sim {\mathbb Z}^\ast_p}[H({\mathbf w},x)t_1 (a\cdot x)]^2\big].
\end{split}
\end{equation*}
where $H({\mathbf w},x) \triangleq \partial_{w_i}\eta({\mathbf w},x)
\frac{l'(\eta({\mathbf w},x))-l'(-\eta({\mathbf w},x))}{2}$. Let us fix ${\mathbf w}$ and treat $H({\mathbf w},x)$ as a vector from ${\mathbb R}^p$. Using the Boas-Bellman inequality we bound the latter expression by $\frac{4\|H({\mathbf w},x)\|^2}{(p-1)^3}\big(\max_{a\in {\mathbb Z}_p^\ast} \|t_1(a\cdot x)\|^2+\sqrt{\sum_{a\ne b}\langle t_1(a\cdot x),t_1(b\cdot x)\rangle^2}\big)$. Using 1-Lipschitzness of $l$, we have $\|H({\mathbf w},x)\|^2\leq \|g_i({\mathbf w})\|^2$. Further, we proceed identically to the proof of the squared loss case.
\end{proof}
Next, our goal will be to find conditions under which $f(Y)$ is concentrated around its mean, whose value is established by the following lemma.

\begin{lemma}\label{th:mean} For $f$ defined according to~\eqref{eq:f} we have ${\mathbb E}_{Y\sim {\mathbb Z}^\ast_p}[f(Y)] = {\mathbb E}_{X\sim {\mathbb Z}^\ast_p} [t(X)]^2$.
\end{lemma}
\begin{proof} The following chain of identities can be easily verified.
\begin{equation*}
\begin{split}
 {\mathbb E}_{Y\sim {\mathbb Z}^\ast_p}[f(Y)] = {\mathbb E}_{Y\sim {\mathbb Z}^\ast_p} {\mathbb E}_{X\sim {\mathbb Z}^\ast_p}[t(X)t(Y\cdot X)]\\ = 
 {\mathbb E}_{X\sim {\mathbb Z}^\ast_p} {\mathbb E}_{Y\sim {\mathbb Z}^\ast_p}[t(X)t(Y\cdot X)]={\mathbb E}_{X\sim {\mathbb Z}^\ast_p} {\mathbb E}_{Y\sim {\mathbb Z}^\ast_p}[t(X)t(Y)] = {\mathbb E}_{X\sim {\mathbb Z}^\ast_p} [t(X)]^2.
\end{split}
\end{equation*}
\end{proof}

\subsection{A key theorem}

The goal of this subsection is to study the distribution of the random variable $f(Y)$, defined by \eqref{eq:f}, where $Y$ is sampled uniformly at random from ${\mathbb Z}_p^\ast$. 
Let us now study the second moment of $f(Y)$ (which is equal to the variance of $f(Y)$). 

\begin{lemma}\label{spectr} Let $\mathbf{\Phi}\triangleq [t(y\cdot x)]_{(x,y)\in ({\mathbb Z}_p^\ast)^2}$ and ${\mathbf t}\triangleq [t(x)]_{x\in {\mathbb Z}_p^\ast}$.
Then,
$ {\mathbb E}[f(Y)^2] = \frac{1}{(p-1)^3}\|\mathbf{\Phi}{\mathbf t}\|^2$.
\end{lemma}
\begin{proof}
The second moment equals
\begin{equation*}
\begin{split}
 {\mathbb E}[f(Y)^2] = \frac{1}{p-1}\sum_{y\in {\mathbb Z}_p^\ast}\frac{1}{(p-1)^2}\left(\sum_{x\in {\mathbb Z}_p^\ast}t(x)t(x\cdot y)\right)^2 \\
=\frac{1}{(p-1)^3}\sum_{(y,x,x')\in ({\mathbb Z}_p^\ast)^3}t(x)t(x\cdot y)t(x')t(x'\cdot y) = \frac{1}{(p-1)^3} \|[\sum_{x\in {\mathbb Z}_p^\ast} t(y\cdot x)t(x)]_{y\in {\mathbb Z}_p^\ast}\|^2 
=\frac{1}{(p-1)^3}\|\mathbf{\Phi}{\mathbf t}\|^2.
\end{split}
\end{equation*}
\end{proof}
Our next goal will be to study $\|\mathbf{\Phi}{\mathbf t}\|^2$. Let us denote the vector $[t(x)]_{x\in {\mathbb Z}_p}\in {\mathbb R}^{p}$ by ${\mathbf a}$. Let $\omega=e^{\frac{2\pi {\rm i}}{p}}$ be a primitive $p$th root of unity. Other primitive roots of unity are $\omega_2, \cdots, \omega_{p-1}$ where $\omega_k = \omega^k, k\in {\mathbb Z}_p$. The matrix
$\frac{1}{\sqrt{p}}\mathbf{U}_k$,
where $\mathbf{U}_k = [\omega^{ij}_k]_{i,j\in {\mathbb Z}_p}$,
is unitary for $k\in {\mathbb Z}_p^\ast$. In fact, $\mathbf{U}^\top_1$ is a discrete Fourier transform (DFT) matrix. For any $h: {\mathbb Z}_p\to {\mathbb C}$, its DFT is defined by
\begin{equation*}
\begin{split}
\widehat{h}(y) = \sum_{x\in {\mathbb Z}_p}\omega^{-yx}h(x).
\end{split}
\end{equation*}
Recall that $\frac{1}{\sqrt{p}}\mathbf{U}_k$ is a unitary matrix. Let us denote $\mathbf{U}_1=\begin{bmatrix}
{\mathbf b}_0, \cdots, {\mathbf b}_{p-1}
\end{bmatrix}$. Our key tool for bounding $\|\mathbf{\Phi}{\mathbf t}\|^2$ is the following theorem.
\begin{theorem}\label{th:a} Suppose that $t(0)=\sum_{x\in {\mathbb Z}_p^\ast}t(x)=0$. Let $t(x)=t^1(x)+t^2(x)$ where $t_2(x)=c[x\ne 0]$. Then, we have 
\begin{equation*}
\begin{split}
 {\mathbb E}[f(Y)^2]\leq
\frac{1}{p(p-1)^3}(\sum_{x\in {\mathbb Z}_p^\ast}|\widehat{t^1}(x)|)^2(\sum_{x\in {\mathbb Z}_p^\ast}|t(x)|^2).
\end{split}
\end{equation*}
\end{theorem} 
\begin{proof} From unitarity of ${\mathbf U}_1$, we conclude that $\left\{{\mathbf e}_i=\frac{1}{\sqrt{p}}{\mathbf b}_i\right\}_{i=0}^{p-1}$ is an orthonormal basis in ${\mathbb C}^{p}$. The inverse DFT can be understood as an expansion
\begin{equation*}
\begin{split}
{\mathbf a} = \sum_{i=0}^{p-1}({\mathbf e}_i^\dag {\mathbf a}){\mathbf e}_i.
\end{split}
\end{equation*}
Note that
\begin{equation*}
\begin{split}
({\mathbf e}_i^\dag {\mathbf a})&=\frac{1}{\sqrt{p}}\sum_{x=0}^{p-1}\omega^{-xi}t(x) = \frac{1}{\sqrt{p}}\widehat{t}(i).
\end{split}
\end{equation*}
Thus, we conclude that 
\begin{equation*}
\begin{split}
{\mathbf a} = \sum_{i=0}^{p-1} \frac{1}{\sqrt{p}}\widehat{t}(i) {\mathbf e}_i= 
\frac{1}{p}\sum_{i=0}^{p-1} \widehat{t}(i){\mathbf b}_i.
\end{split}
\end{equation*}
Note that $\widehat{t}(0)=0$ due to $\sum_{x\in {\mathbb Z}_p}t(x)=0$.
From the latter equation we conclude that $t(x) = \frac{1}{p}\sum_{k=1}^{p-1}\widehat{t}(k) \omega^{kx}$, and therefore $t(x\cdot y) = \frac{1}{p}\sum_{k=1}^{p-1}\widehat{t}(k) \omega^{kxy}$. Equivalently,
\begin{equation*}
\begin{split}
\mathbf{\Psi}=[t(x\cdot y)]_{(x,y)\in {\mathbb Z}_p^2}=\frac{1}{p}\sum_{k=1}^{p-1}\widehat{t}(k) {\mathbf U}_{k}.
\end{split}
\end{equation*}
Using $t(0)=0$, we have
\begin{equation*}
\begin{split}
\mathbf{\Psi}{\mathbf a}=t(0)^2\begin{bmatrix} 
1\\
\vdots\\
1\end{bmatrix} +\begin{bmatrix} 
t(0) \sum_{x\in {\mathbb Z}_p^\ast} t(x) \\
\mathbf{\Phi}{\mathbf t}\\
\end{bmatrix}=\begin{bmatrix} 
0 \\
\mathbf{\Phi}{\mathbf t}\\
\end{bmatrix}.
\end{split}
\end{equation*}
We again use $t(0)=0$ in the following chain of equations:
\begin{equation*}
\begin{split}
\|\mathbf{\Phi}{\mathbf t}\|^2=\|\mathbf{\Psi}{\mathbf a}\|^2 =\frac{1}{p^2}
\|\sum_{k=1}^{p-1}\widehat{t}(k)\mathbf{U}_k{\mathbf a}\|^2= \frac{1}{p^2}\|[\sum_{k=1}^{p-1}\sum_{s=0}^{p-1} \omega^{iks} \widehat{t}(k) t(s)]_{i\in {\mathbb Z}_p}\|^2\\ =
\frac{1}{p^2}\|\sum_{k=1}^{p-1}\sum_{s=1}^{p-1}\widehat{t}(k) t(s) {\mathbf b}_{k\cdot s}\|^2 =
\frac{1}{p}\sum_{a\in {\mathbb Z}_p^\ast}|\sum_{s\in {\mathbb Z}_p^\ast} \widehat{t}(a/s) t(s)|^2.
\end{split}
\end{equation*}
Let us denote the convolution on the multiplicative group $({\mathbb Z}^\ast_p, \cdot)$ by $\ast$, i.e.  $(x\ast y)(a) = \sum_{s\in {\mathbb Z}_p^\ast} x(a/s) y(s)$.
Note that
\begin{equation*}
\begin{split}
\big(\sum_{a\in {\mathbb Z}_p^\ast}|\sum_{s\in {\mathbb Z}_p^\ast} \widehat{t}(a/s) t(s)|^2\big)^{1/2} = \|\widehat{t}\ast t\|_2\leq \|\widehat{t^1}\ast t\|_2+\|\widehat{t^2}\ast t\|_2.
\end{split}
\end{equation*}
where $\|\alpha\|_q=(\sum_{x\in {\mathbb Z}_p^\ast}|\alpha(x)|^q)^{1/q}$. 
The fact that $\widehat{t^2}\ast t =0$ can be shown using an inverse sequence of equations: 
\begin{equation*}
\begin{split}
\|\widehat{t^2}\ast t\|^2_2 = \sum_{a\in {\mathbb Z}_p^\ast}|\sum_{s\in {\mathbb Z}_p^\ast} \widehat{t^2}(a/s) t(s)|^2 =\frac{1}{p} \|\sum_{k=1}^{p-1}\sum_{s=1}^{p-1}\widehat{t^2}(k) t(s) {\mathbf b}_{k\cdot s}\|^2 \\ = 
\frac{1}{p}\|[\sum_{k=1}^{p-1}\sum_{s=0}^{p-1} \omega^{iks} \widehat{t^2}(k) t(s)]_{i\in {\mathbb Z}_p}\|^2 =\frac{1}{p}\|\sum_{k=1}^{p-1}\widehat{t^2}(k)\mathbf{U}_k{\mathbf a}\|^2.
\end{split}
\end{equation*}
We have $\mathbf{U}_0{\mathbf a}={\mathbf 0}$. Therefore, $\sum_{k=1}^{p-1}\widehat{t^2}(k)\mathbf{U}_k{\mathbf a} = \sum_{k=0}^{p-1}\widehat{t^2}(k)\mathbf{U}_k{\mathbf a} = [t^2(x\cdot y)]_{x,y\in {\mathbb Z}_p}{\mathbf a}={\mathbf 0}$. Thus, $\widehat{t^2}\ast t=0$.

By the Young Convolution Theorem~\cite{Barthe1998}, we have
\begin{equation*}
\begin{split}
\|\widehat{t^1}\ast t\|_2\leq \|\widehat{t^1}\|_{q} \|t\|_{q'},
\end{split}
\end{equation*}
where $\frac{1}{q}+\frac{1}{q'}=\frac{3}{2}$. Let us set $q=1$ and $q'=2$. Then,
\begin{equation*}
\begin{split}
\|\mathbf{\Phi}{\mathbf t}\|^2\leq 
\frac{1}{p}\big((\sum_{x\in {\mathbb Z}_p^\ast}|\widehat{t^1}(x)|)(\sum_{x\in {\mathbb Z}_p^\ast}|t(x)|^2)^{1/2}+0\big)^2=
\frac{1}{p}(\sum_{x\in {\mathbb Z}_p^\ast}|\widehat{t^1}(x)|)^2(\sum_{x\in {\mathbb Z}_p^\ast}|t(x)|^2).
\end{split}
\end{equation*}
Finally, using Lemma~\ref{spectr} we obtain the needed statement.
\end{proof}

\section{Applications of Theorem~\ref{th:a} to modular multiplication}\label{applications}
Before we start applying Theorem~\ref{th:a} we need one standard lemma.
Let us first bound the sums $\sum_{k=0}^{p-1}\frac{1}{|1+\omega^{rk}|}$, $\sum_{k=1}^{p-1}\frac{1}{|1-\omega^{rk}|}$. 
\begin{lemma}\label{simple-case} We have $\sum_{k=0}^{p-1}\frac{1}{|1+\omega^{rk}|}\Lt p\log p$ and $\sum_{k=1}^{p-1}\frac{1}{|1-\omega^{rk}|}\Lt p\log p$ for any $r\in {\mathbb Z}_p^\ast$.
\end{lemma}
\begin{proof} Since $\sum_{k=0}^{p-1}\frac{1}{|1+\omega^{rk}|}=\sum_{k=0}^{p-1}\frac{1}{|1+\omega^{k}|}$, it is enough to prove the first statement for $r=1$.

Let $\theta=\frac{2\pi k}{p}$. Note  that $\theta\in [0,2\pi)$ and
\begin{equation*}
\begin{split}
|1+\omega^{k}| = |1+e^{{\rm i}\theta}| = (2+2\cos(\theta))^{1/2}=2\left|\cos\left(\frac{\theta}{2}\right)\right|.
\end{split}
\end{equation*}
Let us denote $2\psi=\theta-\pi$. Thus, $|1+\omega^{k}| =2\left|\cos\left(\frac{2\psi+\pi}{2}\right)\right| = 2|\sin(\psi)|\geq |\psi|$ if $\psi\in \left[-\frac{\pi}{4}, \frac{\pi}{4}\right]$. Note that $\psi\in \left[-\frac{\pi}{4}, \frac{\pi}{4}\right]$ if and only if $-\frac{\pi}{4}\leq \frac{\pi k}{p}-\frac{\pi}{2}\leq \frac{\pi}{4}$, or $\frac{1}{4}p\leq k\leq \frac{3}{4}p$. Thus, we have 
\begin{equation*}
\begin{split}
\sum_{k\in \left[\frac{1}{4}p, \frac{3}{4}p\right]\cap {\mathbb Z}_p}\frac{1}{|1+\omega^{k}|}&\leq \sum_{k\in \left[\frac{1}{4}p, \frac{3}{4}p\right]\cap {\mathbb Z}_p}\frac{1}{\left|\frac{\pi k}{p}-\frac{\pi}{2}\right|}=\frac{2p}{\pi}\sum_{k\in \left[\frac{1}{4}p, \frac{3}{4}p\right]\cap {\mathbb Z}_p}\frac{1}{|2k-p|} 
\leq\frac{4p}{\pi}\sum_{i=1}^{\lceil p/2 \rceil}\frac{1}{i} \Lt p\log p.
\end{split}
\end{equation*}
Since $\left|1+\omega^{k}\right| = 2|\sin(\psi)|\geq 1$ if $\psi\in \left[-\frac{\pi}{2},-\frac{\pi}{4}\right]\cup \left[\frac{\pi}{4}, \frac{\pi}{2}\right]$, then
\begin{equation*}
\begin{split}
\sum_{k\in {\mathbb Z}_p:\,\,\psi\in\left[-\frac{\pi}{2},-\frac{\pi}{4}\right]\cup \left[\frac{\pi}{4}, \frac{\pi}{2}\right]}\frac{1}{\left|1+\omega^{k}\right|}\asymp p.
\end{split}
\end{equation*}
Thus, the total sum satisfies
\begin{equation*}
\begin{split}
\sum_{k\in {\mathbb Z}_p}\frac{1}{\left|1+\omega^{k}\right|}\Lt p\log p.
\end{split}
\end{equation*}
Let us now prove the second statement, i.e. $\sum_{k=1}^{p-1}\frac{1}{|1-\omega^{rk}|}\Lt p\log p$. Again, it is enough to prove it for $r=1$. Since $\sum_{k=1}^{p-1}\frac{1}{|1-\omega^{k}|}=\sum_{k=-\frac{p-1}{2}}^{-1}\frac{1}{|1-\omega^{k}|}+\sum_{k=1}^{\frac{p-1}{2}}\frac{1}{|1-\omega^{k}|}$, let us prove first that $\sum_{k=1}^{\frac{p-1}{2}}\frac{1}{|1-\omega^{k}|}\Lt p\log p$.

Let $\theta=\frac{2\pi k}{p}, 1\leq k\leq \frac{p-1}{2}$. Note  that $\theta\in (0,\pi)$ and
\begin{equation*}
\begin{split}
|1-\omega^{k}| = |1-e^{{\rm i}\theta}| = (2-2\cos(\theta))^{1/2}=2\left|\sin\left(\frac{\theta}{2}\right)\right|.
\end{split}
\end{equation*}
Let us denote $2\psi=\theta$. The condition $0 < \psi\leq \frac{\pi}{4}$ is equivalent to $1\leq k\leq \frac{p}{4}$. Under that condition, we have $2|\sin(\psi)|\geq |\psi|$. Therefore, we have
\begin{equation*}
\begin{split}
\sum_{k\in \left[1, \frac{p}{4}\right]\cap {\mathbb Z}_p}\frac{1}{|1-\omega^{k}|}&\leq \sum_{k\in \left[1, \frac{p}{4}\right]\cap {\mathbb Z}_p}\frac{1}{\left|\frac{\pi k}{p}\right|}=\frac{p}{\pi}\sum_{k\in \left[1, \frac{p}{4}\right]\cap {\mathbb Z}_p}\frac{1}{k}  \Lt p\log p.
\end{split}
\end{equation*}
Also, $2|\sin(\psi)|\geq 1, \psi\in [\frac{\pi}{4}, \frac{\pi}{2})$, therefore
\begin{equation*}
\begin{split}
\sum_{k\in \left[\frac{p}{4}, \frac{p-1}{2}\right]\cap {\mathbb Z}_p}\frac{1}{|1-\omega^{k}|}&\leq \sum_{k\in \left[1, \frac{p}{4}\right]\cap {\mathbb Z}_p}\frac{1}{1}\Lt p.
\end{split}
\end{equation*}
Thus, $\sum_{k=1}^{\frac{p-1}{2}}\frac{1}{|1-\omega^{k}|}\Lt p\log p$. Analogously one can prove $\sum_{k=-\frac{p-1}{2}}^{-1}\frac{1}{|1-\omega^{k}|}\Lt p\log p$.
\end{proof}

\subsection{Application to the whole output}\label{whole-out}
Let us now consider the case of 
$$t(x)=x-\frac{p}{2}, x\in {\mathbb Z}_p^\ast, $$
and $t(0)=0$. By construction ${\mathbb E}_{Y\sim {\mathbb Z}_p^\ast}[t(Y)]=0$. Let us define $f$ according to the equation~\eqref{eq:f}, i.e. by $f(y)={\mathbb E}_{X\sim {\mathbb Z}_p^\ast}t(X)t(X\cdot y)$.  From Lemma~\ref{th:mean} we have ${\mathbb E}[f(Y)]=0$. The variance can be bounded by the following lemma.

\begin{lemma} For $t(x)=\big(x-\frac{p}{2}\big)[x\ne 0]$ and $f(y)={\mathbb E}_{X\sim {\mathbb Z}_p^\ast}t(X)t(X\cdot y)$,
we have
\begin{equation*}
\begin{split}
{\mathbb E}_{Y\sim {\mathbb Z}_p^\ast}[f(Y)^2]\leq C p^3\log^2 p .
\end{split}
\end{equation*}
where $C$ is some universal constant.
\end{lemma}
\begin{proof} Let us denote $t^1(x) = x, x\in {\mathbb Z}_p$, $t^2(x)=-\frac{p}{2}[x\ne 0]$ and verify that $t^1,t^2$ satisfy conditions of Theorem~\ref{th:a}.
Since $\sum_{x\in {\mathbb Z}_p^\ast}|t(x)|^2\leq 2\sum_{i=1}^{\frac{p-1}{2}}i^2\Lt p^3$, then  Theorem~\ref{th:a} gives us
\begin{equation*}
\begin{split}
{\mathbb E}_{Y\sim {\mathbb Z}_p^\ast}[f(Y)^2]\Lt \frac{1}{p}(\sum_{x\in {\mathbb Z}_p^\ast}|\widehat{t^1}(x)|)^2.
\end{split}
\end{equation*}
A direct calculation of $\widehat{t^1}$ gives us 
\begin{equation*}
\begin{split}
\widehat{t^1}(i)=\sum_{x\in {\mathbb Z}_p^\ast}\omega^{-ix}x.
\end{split}
\end{equation*}
Using $\sum_{x\in {\mathbb Z}_p^\ast}\omega^{-ix}x = -\frac{d}{d(2\pi {\rm i}i/p)}(\sum_{x\in {\mathbb Z}_p^\ast}\omega^{-ix})=-\frac{p}{2\pi {\rm i}}\frac{d}{di}(\frac{1-\omega^{-pi}}{1-\omega^{-i}}-1)=\frac{p}{2\pi {\rm i}}\big(\frac{\frac{2\pi {\rm i}}{p}\omega^{-i}}{(1-\omega^{-i})^2}-\frac{2\pi {\rm i}}{1-\omega^{-i}}\big)$, we conclude
\begin{equation*}
\begin{split}
\widehat{t^1}(i)=
-\frac{p}{1-\omega^{-i}}+\frac{\omega^{-i}}{(1-\omega^{-i})^2}.
\end{split}
\end{equation*}
Therefore,
\begin{equation*}
\begin{split}
\sum_{i\in {\mathbb Z}_p^\ast}|\widehat{t^1}(i)|\leq
p\sum_{i\in {\mathbb Z}_p^\ast} \frac{1}{|1-\omega^{-i}|}+\frac{1}{|1-\omega^{-i}|^2}.
\end{split}
\end{equation*}
From Lemma~\ref{simple-case} we have $\sum_{i\in {\mathbb Z}_p^\ast} \frac{1}{|1-\omega^{-i}|}\Lt p\log p$ and 
\begin{equation*}
\begin{split}
\sum_{i\in {\mathbb Z}_p^\ast} \frac{1}{|1-\omega^{-i}|^2}=2\sum_{i=1}^{\frac{p-1}{2}} \frac{1}{\sin^2(\frac{\pi i}{p})}
= 2\sum_{i=1}^{\lfloor p/4\rfloor} \frac{1}{\sin^2(\frac{\pi i}{p})}+2\sum_{i=\lfloor p/4\rfloor+1}^{\frac{p-1}{2}} \frac{1}{\sin^2(\frac{\pi i}{p})} \Lt \sum_{i=1}^{\lfloor p/4\rfloor} \frac{1}{(\frac{\pi i}{p})^2} + p\Lt  p^2.
\end{split}
\end{equation*}
Therefore,
\begin{equation*}
\begin{split}
{\mathbb E}[f(Y)^2]\Lt \frac{1}{p}(p^2\log p)^2\Lt p^3\log^2 p .
\end{split}
\end{equation*}
Lemma proved.
\end{proof}
Proved lemma directly leads to the following corollary.
\begin{corollary}\label{add-exper} We have
\begin{equation*}
\begin{split}
\E_{A,B\sim\mathbb{Z}_p^\ast}\left(\Cov_{X\sim\mathbb{Z}_p^\ast}[A\cdot X, B\cdot X]\right)^2 \Lt p^3\log^2 p .
\end{split}
\end{equation*}
\end{corollary}
\begin{proof} Note that $\Cov_{X\sim\mathbb{Z}_p^\ast}[A\cdot X, B\cdot X] = \E_{X\sim\mathbb{Z}_p^\ast} [\mathring{h}_A(X)\mathring{h}_B(X)]$ where $\mathring{h}_a = t(a\cdot x)$ for $t(x)=\big(x-\frac{p}{2}\big)[x\ne 0]$. 
The statement follows from a combination of Lemma~\ref{lem:sum_of_squares} and the last Lemma.
\end{proof}
\begin{theorem}\label{standard}
For standardized $t(x)=\frac{(a\cdot x)-\frac{p}{2}}{\sqrt{\frac{p^2}{12}-\frac{p}{6}}}[x\ne 0]$ we have
\begin{equation*}
\begin{split}
{\rm Var}_{a\sim {\mathbb Z}^\ast_p}\big[\nabla_{\mathbf w} {\mathbb E}_{x\sim {\mathbb Z}^\ast_p}[(t(a\cdot x)- \eta({\mathbf w},x))^2]\big]  \Lt \|g({\mathbf w})\|^2_{\ast s} \frac{\log p}{\sqrt{p}} .
\end{split}
\end{equation*}
\end{theorem}
\begin{proof}
Using the previous theorem we conclude that for standardized $t(x)=\frac{(a\cdot x)-\frac{p}{2}}{\sqrt{\frac{p^2}{12}-\frac{p}{6}}}[x\ne 0]$ we have ${\mathbb E}[f(Y)^2]\Lt  \frac{p^3\log^2 p}{p^4}$. From Theorem~\ref{BBBB2} we conclude that
$$
{\rm Var}_{a\sim {\mathbb Z}^\ast_p}\big[\nabla_{\mathbf w} {\mathbb E}_{x\sim {\mathbb Z}^\ast_p}[(t(a\cdot x)- \eta({\mathbf w},x))^2]\big]  \Lt \|g({\mathbf w})\|^2_{\ast s} \sqrt{\frac{\log^2 p}{p}}.
$$
\end{proof}

\subsection{Application: learning specific bits}
Let $[x]_r$ denote $x_r$ for $x=\sum_{i=1}^n x_i2^{i-1}$, $x_i\in \{0,1\}$. I.e.  $[x]_r$ is the $r$th bit from the end of the binary encoding of $x$. Also, let us denote $p-2^{r-1}\lfloor \frac{p}{2^{r-1}}\rfloor$ by $p_{r-1}$. Note that $0\leq p_{r-1}<2^{r-1}$ and $p_{r-1}$ is just a number whose binary encoding forms the last $r-1$ bits of $p$.

Let us denote 
\begin{equation}
t_r(x) = (-1)^{[x]_r}, x\in {\mathbb Z}_p^\ast
\end{equation}
 and $t_r(0)=0$. 

The main result of this subsection is the following theorem.
\begin{theorem}\label{r-var}
For $\mathring{\mathcal{H}}_p = \{t_r(a\cdot x)\mid a\in {\mathbb Z}^\ast_p\}$ and $L(y,y')=l(yy')$ where $l$ is 1-Lipschitz, we have
\begin{equation*}
\begin{split}
{\rm Var}_{a\sim {\mathbb Z}^\ast_p}\big[\nabla_{\mathbf w} {\mathbb E}_{x\sim {\mathbb Z}^\ast_p}L(t_r(a\cdot x),\eta({\mathbf w},x))]\big]  \Lt \|g({\mathbf w})\|^2_{\ast s} \frac{r(\log_2 p+1-r)}{\sqrt{p}} .
\end{split}
\end{equation*}
\end{theorem}
Our proof will be given in the end and it is based on the sequence of the following statements. 
\begin{lemma}\label{lem:mean-r} ${\E}_{Y\sim {\mathbb Z}_p^\ast}[t_r(Y)] = \frac{-1+(-1)^{[p]_r}p_{r-1}+2^{r-1}[p]_r}{p-1}$ for $r>1$ and ${\E}_{Y}[t_1(Y)] = 0$.
\end{lemma}
\begin{proof}
We have
\begin{equation*}
{\E}_{Y}[t_r(Y)]=\frac{1}{p-1}\sum_{y\in {\mathbb Z}_p^\ast} (-1)^{[x]_r}=\frac{1}{p-1}\sum_{x\in {\mathbb Z}_{p}^\ast} (-1)^{\lfloor x/2^{r-1}\rfloor}.
\end{equation*}
Since $\sum_{x\in {\mathbb Z}_{p}}  (-1)^{\lfloor x/2^{r-1}\rfloor}$ consists of alternating segments of length $2^{r-1}$ and
\begin{equation*}
\sum_{x=0}^{p-1}  (-1)^{\lfloor x/2^{r-1}\rfloor} = \left\{\begin{array}{lr}
        p_{r-1}, & \text{for even } \lfloor \frac{p}{2^{r-1}}\rfloor\\
        2^{r-1}-p_{r-1}, & \text{for odd } \lfloor \frac{p}{2^{r-1}}\rfloor\\
        \end{array}\right.
\end{equation*}
we conclude that $\sum_{x\in {\mathbb Z}_{p}^\ast}  (-1)^{\lfloor x/2^{r-1}\rfloor} =-1+(-1)^{[p]_r}p_{r-1}+2^{r-1}[p]_r$.
\end{proof}
Let us now denote the centered version of $t_r$ by $\tilde{t}_r$:
\begin{equation}
\tilde{t}_r(x)=t_r(x)-\frac{-1+(-1)^{[p]_r}p_{r-1}+2^{r-1}[p]_r}{p-1}, x\in {\mathbb Z}_p^\ast,
\end{equation}
and $\tilde{t}_r(0)=0$. Thus, we have
$$
\tilde{t}_r(x)=t_r(x)  +c_r[x\ne 0],
$$
where $c_r = -\frac{-1+(-1)^{[p]_r}p_{r-1}+2^{r-1}[p]_r}{p-1}$.

Let us denote $\tilde{f}_r(y)= {\mathbb E}_{X\sim {\mathbb Z}_{p}^\ast}\tilde{t}_r(X)\tilde{t}_r(y\cdot X)$.
From Lemma~\ref{th:mean} we conclude that 
\begin{equation}
{\E}_{Y}[\tilde{f}_r(Y)] = 0.\label{eq:mean_f-r}
\end{equation}
Let us now bound the variance of $\tilde{f}_r(Y)$.
\begin{lemma}\label{r-bit-bound} We have
\begin{equation*}
\begin{split}
{\mathbb E}[\tilde{f}_r(Y)^2]\leq 
C\frac{r^2 (\log_2 p+1-r)^2}{p}.
\end{split}
\end{equation*}
where $C$ is some universal constant.
\end{lemma}
 To apply Theorem~\ref{th:a} we need to compute the DFT of $t_r$ first. 
\begin{lemma}\label{lem:a} We have 
\begin{equation*}
\begin{split}
\widehat{t_r}(i) = -1+\frac{1-(-1)^{[p]_r}\omega^{-(p-p_{r-1})i}}{1+\omega^{-2^{r-1}i}}  \frac{1-\omega^{-2^{r-1}i}}{1-\omega^{-i}} + (-1)^{[p]_r}\frac{\omega^{-(p-p_{r-1}) i}-\omega^{-pi}}{1-\omega^{-i}},
\end{split}
\end{equation*}
if $i\in {\mathbb Z}_p^\ast$ and $\widehat{t_r}(0)=-1+2^{r-1}[p]_r+(-1)^{[p]_r}p_{r-1}$.
\end{lemma} 
\begin{proof} 
Note that
\begin{equation*}
\begin{split}
\widehat{t_r}(i)&=\sum_{x=1}^{p-1}\omega^{-xi}(-1)^{[x]_r} = \sum_{x=1}^{2^{r-1}\lfloor \frac{p}{2^{r-1}}\rfloor-1}\omega^{-xi}(-1)^{[x]_r}+\sum_{x=2^{r-1}\lfloor \frac{p}{2^{r-1}}\rfloor}^{p-1}\omega^{-xi}(-1)^{[x]_r}.
\end{split}
\end{equation*}
The first term equals 
\begin{equation*}
\begin{split}
-1+\sum_{x=0}^{2^{r-1}\lfloor \frac{p}{2^{r-1}}\rfloor-1}\omega^{-xi}(-1)^{[x]_r} =-1+\sum_{t=0}^{\lfloor \frac{p}{2^{r-1}}\rfloor-1}(-1)^t\omega^{-2^{r-1}t i}\sum_{x=0}^{2^{r-1}-1}\omega^{-xi} \\ =
-1+\frac{1-(-1)^{[p]_r}\omega^{-(p-p_{r-1})i}}{1+\omega^{-2^{r-1}i}}  \frac{1-\omega^{-2^{r-1}i}}{1-\omega^{-i}},
\end{split}
\end{equation*}
if $i\ne 0$. If $i=0$, then it equals $-1+2^{r-1}[p]_r$.
Whereas,  the second term equals
\begin{equation*}
\begin{split}
(-1)^{[p]_r}\sum_{x=p-p_{r-1}}^{p-1}\omega^{-xi} = (-1)^{[p]_r}\frac{\omega^{-(p-p_{r-1}) i}-\omega^{-pi}}{1-\omega^{-i}},
\end{split}
\end{equation*}
if $i\ne 0$. If $i=0$, then it equals $(-1)^{[p]_r}p_{r-1}$.
Thus, we conclude that 
\begin{equation*}
\begin{split}
\widehat{t_r}(i) = \big(-1+\frac{1-(-1)^{[p]_r}\omega^{-(p-p_{r-1})i}}{1+\omega^{-2^{r-1}i}}  \frac{1-\omega^{-2^{r-1}i}}{1-\omega^{-i}} + (-1)^{[p]_r}\frac{\omega^{-(p-p_{r-1}) i}-\omega^{-pi}}{1-\omega^{-i}}\big)  [i\ne 0]+\\
(-1+2^{r-1}[p]_r+(-1)^{[p]_r}p_{r-1})[i=0].
\end{split}
\end{equation*}
\end{proof}
An application of Theorem~\ref{th:a} requires an estimate of the sum $\sum_{i\in {\mathbb Z}_p^\ast}|\widehat{t_r}(i)|$ which is made in the following lemma.  
\begin{lemma} We have $\sum_{i\in {\mathbb Z}_p^\ast}|\widehat{t_r}(i)| \Lt p r(\log_2 p+1-r)$.
\end{lemma}
\begin{proof} First let us consider the case of $[p]_r=0$. In that case, we have
\begin{equation*}
\begin{split}
\sum_{i\in {\mathbb Z}_p^\ast}|\widehat{t_r}(i)| \leq p+ \sum_{i\in {\mathbb Z}_p^\ast} \frac{|1-\omega^{-2^{r-1}\lfloor \frac{p}{2^{r-1}}\rfloor i}|}{|1+\omega^{-2^{r-1}i}|}  \frac{|1-\omega^{-2^{r-1}i}|}{|1-\omega^{-i}|} + \frac{2}{|1-\omega^{-i}|}.\\
\end{split}
\end{equation*}
From Lemma~\ref{simple-case} we have $\sum_{i\in {\mathbb Z}_p^\ast}\frac{2}{|1-\omega^{-i}|}\Lt p\log p$. Thus, it remains to bound the sum of terms $a_i = \frac{|1-\omega^{-2^{r-1}\lfloor \frac{p}{2^{r-1}}\rfloor i}|}{|1+\omega^{-2^{r-1}i}|}  \frac{|1-\omega^{-2^{r-1}i}|}{|1-\omega^{-i}|}$. Using $\sum_{i\in {\mathbb Z}_p^\ast} a_i= \sum_{i=1}^{\frac{p-1}{2}}a_i+\sum_{i=-\frac{p-1}{2}}^{-1}a_i$ and $\frac{1}{|1-\omega^{-i}|}\Lt 1$ for $\frac{p}{4}\leq i \leq \frac{p-1}{2}$ or $-\frac{p-1}{2} \leq i \leq -\frac{p}{4}$ we conclude that $$\sum_{i=\lceil \frac{p}{4}\rceil}^{\frac{p-1}{2}}a_i\Lt 4\sum_{i=\lceil \frac{p}{4}\rceil}^{\frac{p-1}{2}}\frac{1}{|1+\omega^{-2^{r-1}i}|}\Lt p \log p$$ and $\sum_{i=-\frac{p-1}{2}}^{-\lceil \frac{p}{4}\rceil}a_i \Lt p\log p$ (using Lemma~\ref{simple-case}). Thus, a bound of the total sum directly follows from bounds of $\sum_{i=1}^{\lfloor \frac{p}{4}\rfloor}a_i$ (and $\sum_{i=-\lfloor \frac{p}{4}\rfloor}^{1}a_i$). For brevity, let us only show how to bound $S=\sum_{i=1}^{\lfloor \frac{p}{4}\rfloor}a_i$.

Using $|1-\omega^{x}|=2|\sin (\frac{\pi x}{p})|$ and $|1+\omega^{x}|=2|\cos (\frac{\pi x}{p})|$, this sum can be rewritten as
\begin{equation*}
\begin{split}
S= \sum_{i=1}^{\lfloor \frac{p}{4}\rfloor} \frac{|\sin(2^r k\frac{\pi i}{p})|\cdot |\sin(2^{r-1}\frac{\pi i}{p})|}{|\cos(2^{r-1}\frac{\pi i}{p})\sin(\frac{\pi i}{p})|}.
\end{split}
\end{equation*}
where $2k=\lfloor \frac{p}{2^{r-1}}\rfloor $. Recall that $x\pm y = [x-y,x+y]$. Let us denote $n=\lfloor 2^{r-2}\frac{i}{p}\rfloor$. By construction, we have $0\leq n\leq 2^{r-4}-1$. For any $i\in [\lfloor \frac{p}{4}\rfloor]$ at least one of the following inclusions holds
\begin{equation*}
\begin{split}
1)\,\, 2^{r-1}\frac{\pi i}{p}\in \frac{\pi}{2}+2\pi n\pm \frac{\pi}{4}, 2)\,\, 2^{r-1}\frac{\pi i}{p}\in \frac{3\pi}{2}+2\pi n\pm \frac{\pi}{4},{\rm \,\, or}\\
3)\,\, 2^{r-1}\frac{\pi i}{p}\notin (\frac{\pi}{2}+2\pi n\pm \frac{\pi}{4})\cup (\frac{3\pi}{2}+2\pi n\pm \frac{\pi}{4}).
\end{split}
\end{equation*}
In the third case we have $\frac{1}{|\cos(2^{r-1}\frac{\pi i}{p})|}\Lt 1$ and the summation over all such $i$ asymptotically cannot exceed $\sum_{i=1}^{\lfloor \frac{p}{4}\rfloor} \frac{1}{|\sin(\frac{\pi i}{p})|}\Lt \sum_{i=1}^{\lfloor \frac{p}{4}\rfloor} \frac{1}{|\frac{\pi i}{p}|}\Lt p\log p$. The first and the second cases are similar, therefore we will consider only the first one, i.e. we will bound
\begin{equation*}
\begin{split}
\tilde{S} = \sum_{n=0}^{2^{r-4}-1}\sum_{i\in [\lfloor \frac{p}{4}\rfloor]: 2^{r-1}\frac{\pi i}{p}\in \frac{\pi}{2}+2\pi n\pm \frac{\pi}{4}}  \frac{|\sin(2^r k\frac{\pi i}{p})|\cdot |\sin(2^{r-1}\frac{\pi i}{p})|}{|\cos(2^{r-1}\frac{\pi i}{p})\sin(\frac{\pi i}{p})|}.
\end{split}
\end{equation*}
We have $2^{r-1}\frac{\pi i}{p}\in \frac{\pi}{2}+2\pi n\pm \frac{\pi}{4}$, $i\in [\lfloor \frac{p}{4}\rfloor]$ if and only if $i\in [\lfloor \frac{p}{4}\rfloor]\cap \frac{p}{2^r}+\frac{np}{2^{r-2}}\pm \frac{p}{2^{r+1}}$. Let us denote $\varepsilon =  i-(\frac{p}{2^r}+\frac{np}{2^{r-2}})$. From $i\in [\lfloor \frac{p}{4}\rfloor]\cap \frac{p}{2^r}+\frac{np}{2^{r-2}}\pm \frac{p}{2^{r+1}}$ we deduce $\varepsilon \in \pm \frac{p}{2^{r+1}}\cap ({\mathbb Z}-\{\frac{p}{2^r}+\frac{np}{2^{r-2}}\})$ where $\{x\} = x-\lfloor x\rfloor$ and ${\mathbb Z}-s$ denotes $\{z-s\mid z\in {\mathbb Z}\}$. 

Using $ 2|\cos(2^{r-1}\frac{\pi i}{p})|\geq |2^{r-1}\frac{\pi i}{p}-\frac{\pi}{2}-2\pi n| $ for $2^{r-1}\frac{\pi i}{p}\in \frac{\pi}{2}+2\pi n\pm \frac{\pi}{4}$ we obtain
\begin{equation*}
\begin{split}
\tilde{S} \Lt\sum_{n=0}^{2^{r-2}-1}\sum_{i\in [\lfloor \frac{p}{4}\rfloor] \cap \frac{p}{2^r}+\frac{np}{2^{r-2}}\pm \frac{p}{2^{r+1}}} \frac{|\sin(2^r k\frac{\pi i}{p})|\cdot |\sin(2^{r-1}\frac{\pi i}{p})|}{|(2^{r-1}\frac{\pi i}{p}-\frac{\pi}{2}-2\pi n)\frac{\pi i}{p}|} \\ \leq
\sum_{n=0}^{2^{r-2}-1}\sum_{i\in [\lfloor \frac{p}{4}\rfloor] \cap \frac{p}{2^r}+\frac{np}{2^{r-2}}\pm \frac{p}{2^{r+1}}} \frac{|\sin(2^r k\frac{\pi i}{p})|}{|(2^{r-1}\frac{\pi i}{p}-\frac{\pi}{2}-2\pi n)\frac{\pi i}{p}|} \\ \leq 
\sum_{n=0}^{2^{r-2}-1}\sum_{\varepsilon\in  \pm \frac{p}{2^{r+1}}\cap ({\mathbb Z}-\{\frac{p}{2^r}+\frac{np}{2^{r-2}}\})}\frac{|\sin(2^r k\frac{\pi \varepsilon}{p})|}{\frac{2^{r-1}\pi}{p}|\varepsilon|(\frac{\pi}{2^r}+\frac{\pi n}{2^{r-2}}+\frac{\pi}{p}\varepsilon)}.
\end{split}
\end{equation*}
Let us denote $q= \frac{\pi 2^r k }{p}$. Note that $\frac{\pi}{2}\leq \frac{\pi 2^r k}{p}\leq \pi$. We have
\begin{equation*}
\begin{split}
\sum_{\varepsilon\in  \pm \frac{p}{2^{r+1}}\cap ({\mathbb Z}-\{\frac{p}{2^r}+\frac{np}{2^{r-2}}\})}\frac{|\sin(q  \varepsilon)|}{\frac{2^{r-1}\pi}{p}|\varepsilon|(\frac{\pi}{2^r}+\frac{\pi n}{2^{r-2}}+\frac{\pi}{p}\varepsilon)} \Lt \frac{B_r}{\frac{2^{r-1}\pi}{p}(\frac{\pi}{2^{r+1}}+\frac{\pi n}{2^{r-2}})},
\end{split}
\end{equation*}
where
\begin{equation*}
\begin{split}
B_r = \max_{s\in [0,1]}\sum_{\varepsilon\in  \pm \frac{p}{2^{r+1}}\cap ({\mathbb Z}-s)} \frac{|\sin(q \varepsilon)|}{|\varepsilon|}.
\end{split}
\end{equation*}
Obviously, we have $B_r\leq \sum_{\varepsilon\in  \pm \frac{p}{2^{r+1}}\cap ({\mathbb Z}-s^\ast)} \frac{|\sin(q \varepsilon)|}{|\varepsilon|}$ for some $s^\ast\in [0,1]$, and the latter is bounded by $2q+2\sum_{i=1}^{\lceil\frac{p}{2^{r+1}}\rceil} \frac{1}{i}\Lt \log(\frac{p}{2^{r+1}}+1)$.

Thus, $\tilde{S}$ is asymptotically bounded by
\begin{equation*}
\begin{split}
pB_r \sum_{n=0}^{2^{r-2}-1}\frac{1}{2^{r-1}(\frac{1}{2^{r+1}}+\frac{n}{2^{r-2}})}\Lt p \log(\frac{p}{2^{r+1}}+1)\log (2^{r-2}+1) \Lt p r(\log_2 p+1-r),
\end{split}
\end{equation*}
and therefore, the total sum is bounded by $p r(\log_2 p+1-r)$.

Let us now consider the case of $[p]_r=1$. As in the previous case, we can reduce bounding the total sum to bounding the sum $U=\sum_{i=1}^{\lfloor \frac{p}{4}\rfloor}b_i$ where $b_i=\frac{|1+\omega^{-2^{r-1}\lfloor \frac{p}{2^{r-1}}\rfloor i}|}{|1+\omega^{-2^{r-1}i}|}  \frac{|1-\omega^{-2^{r-1}i}|}{|1-\omega^{-i}|}$, which is equal to
\begin{equation*}
\begin{split}
U = \sum_{i=1}^{\lfloor \frac{p}{4}\rfloor} \frac{|\cos(2^{r-1} (2k+1)\frac{\pi i}{p})|\cdot |\sin(2^{r-1}\frac{\pi i}{p})|}{|\cos(2^{r-1}\frac{\pi i}{p})\sin(\frac{\pi i}{p})|}.
\end{split}
\end{equation*}
where $2k+1=\lfloor \frac{p}{2^{r-1}}\rfloor$.
As in the previous case, $U$ can be bounded according to
\begin{equation*}
\begin{split}
U  \Lt\sum_{n=0}^{2^{r-2}-1}\sum_{i\in [\lfloor \frac{p}{4}\rfloor] \cap \frac{p}{2^r}+\frac{np}{2^{r-2}}\pm \frac{p}{2^{r+1}}} \frac{|\cos(2^{r-1} (2k+1)\frac{\pi i}{p})|\cdot |\sin(2^{r-1}\frac{\pi i}{p})|}{|(2^{r-1}\frac{\pi i}{p}-\frac{\pi}{2}-2\pi n)\frac{\pi i}{p}|} \\ \leq
\sum_{n=0}^{2^{r-2}-1}\sum_{i\in [\lfloor \frac{p}{4}\rfloor] \cap \frac{p}{2^r}+\frac{np}{2^{r-2}}\pm \frac{p}{2^{r+1}}} \frac{|\cos(2^{r-1}( 2k+1)\frac{\pi i}{p})|}{|(2^{r-1}\frac{\pi i}{p}-\frac{\pi}{2}-2\pi n)\frac{\pi i}{p}|} \\ \leq 
\sum_{n=0}^{2^{r-2}-1}\sum_{\varepsilon\in  \pm \frac{p}{2^{r+1}}\cap ({\mathbb Z}-\{\frac{p}{2^r}+\frac{np}{2^{r-2}}\})}\frac{|\sin(2^{r-1} (2k+1)\frac{\pi \varepsilon}{p})|}{\frac{2^{r-1}\pi}{p}|\varepsilon|(\frac{\pi}{2^r}+\frac{\pi n}{2^{r-2}}+\frac{\pi}{p}\varepsilon)}.
\end{split}
\end{equation*}
The latter sum is asymptotically bounded by $pB'_r \sum_{n=0}^{2^{r-2}-1}\frac{1}{2^{r-1}(\frac{1}{2^{r+1}}+\frac{n}{2^{r-2}})}$ where $B_r'=\max_{s\in [0,1]}\sum_{\varepsilon\in  \pm \frac{p}{2^{r+1}}\cap ({\mathbb Z}-s)} \frac{|\sin(q' \varepsilon)|}{|\varepsilon|}$, $q'=\frac{2^{r-1} (2k+1) \pi}{p}$ and $B'_r\Lt (\log_2 p+1-r)$. Thus, $U\Lt  p r(\log_2 p+1-r)$.
\end{proof}

Now we are ready to apply Theorem~\ref{th:a}.
\begin{proof}[Proof of Lemma~\ref{r-bit-bound}]
Since $\tilde{t}_r(x)=t_r(x)  +c_r[x\ne 0]$, functions $t^1(x)=t_r(x)$ and $t^2(x)=c_r[x\ne 0]$ satisfy all requirements of Theorem~\ref{th:a}. Therefore,
\begin{equation*}
\begin{split}
{\mathbb E}[\tilde{f}_r(Y)^2]\leq
\frac{1}{p(p-1)^3}(\sum_{x\in {\mathbb Z}_p^\ast}|\widehat{t_r}(x)|)^2(\sum_{x\in {\mathbb Z}_p^\ast}|\tilde{t}_r(x)|^2).
\end{split}
\end{equation*}
The previous lemma and $|\tilde{t}_r(x)|\Lt 1$ give us
\begin{equation}\label{BBbound2}
\begin{split}
{\mathbb E}[\tilde{f}_r(Y)^2]\Lt
\frac{1}{p^4}(p r(\log p+1-r))^2 p\Lt \frac{r^2 (\log_2 p+1-r)^2}{p}.
\end{split}
\end{equation}
\end{proof}

As a direct consequence of Lemma~\ref{r-bit-bound} we obtain a proof of Theorem~\ref{r-var}.
\begin{proof}[Proof of Theorem~\ref{r-var}] First we use Theorem~\ref{BBBB21} and bound 
\begin{equation*}
\begin{split}
{\rm Var}_{a\sim {\mathbb Z}^\ast_p}\big[\nabla_{\mathbf w} {\mathbb E}_{x\sim {\mathbb Z}^\ast_p}L(t_r(a\cdot x),\eta({\mathbf w},x))]\big]  \Lt \|g({\mathbf w})\|^2_{\ast s}\sqrt{{\mathbb E}_{Y\sim\mathbb{Z}_p^\ast}[\tilde{f}_r(Y)^2]}.
\end{split}
\end{equation*}
From the latter and Lemma~\ref{r-bit-bound} we obtain the needed bound.
\end{proof}
\section{The SQ dimension and r.h.s. of the Boas-Bellman inequality}\label{rhsBB}
All our bounds are based on the application of the Boas-Bellman inequality. Let $\mathcal{H}\subseteq \{-1,+1\}^{\mathcal{X}}$ be a hypothesis set and $\mathcal{D}$ be a distribution on $\mathcal{X}$. The main part of the r.h.s. of that inequality is the following expression
$$
{\rm BB}(\mathcal{H},\mathcal{D}) \triangleq \sum_{h_1\ne h_2\in \mathcal{H}} {\mathbb E}_{X\sim \mathcal{D}} [h_1(X)h_2(X)]^2. 
$$
The following definition, taken from~\cite{DBLP:conf/stoc/BlumFJKMR94}, plays an important role in the theory of SQ algorithms.
\begin{definition}\label{sq-dim} The statistical query dimension of $\mathcal{H}$ with
respect to $\mathcal{D}$, denoted ${\rm SQ-dim}(\mathcal{H}, \mathcal{D})$, is the largest number $d$ such that $\mathcal{H}$ contains $d$ functions
$f_1, f_2, \cdots, f_d$ such that for all $i\ne j$, $|{\mathbb E}_{X\sim \mathcal{D}}[f_i(X)f_j(X)]| \leq \frac{1}{d}$.
\end{definition}
Let us demonstrate how an upper bound on ${\rm BB}(\mathcal{H},\mathcal{D})$ can be turned to a lower bound on the statistical query dimension of $\mathcal{H}$ w.r.t. $\mathcal{D}$.
\begin{theorem} We have
\begin{equation}
{\rm SQ-dim}(\mathcal{H}, \mathcal{D})+1\geq  \min\big(\frac{|\mathcal{H}|^{2/3}}{2^{1/3}{\rm BB}(\mathcal{H},\mathcal{D})^{1/3}}, |\mathcal{H}|^{1/2}\big).
\end{equation}
\end{theorem}
\begin{proof}
Suppose that ${\rm SQ-dim}(\mathcal{H}, \mathcal{D})= d-1$ for some $d\in {\mathbb N}$. Let us build a simple graph $\mathcal{G}=(\mathcal{H}, \mathcal{E})$, whose set of vertices is $\mathcal{H}$ and a set of edges is defined by $(h_1,h_2)\in \mathcal{E} \Leftrightarrow |{\mathbb E}_{X\sim \mathcal{D}}[h_1(X)h_2(X)]| \leq \frac{1}{d}$. By construction, ${\rm SQ-dim}(\mathcal{H}, \mathcal{D})= d-1$ means that $\mathcal{G}$ does not contain a clique of size $d$. In that case, Turán's theorem~\cite{ALON2016146} implies that
$$
|\mathcal{E}|\leq T(|\mathcal{H}|, d-1),
$$
where $T(|\mathcal{H}|, d-1)$ is the number of edges in the complete $d-1$-partite graph on $|\mathcal{H}|$ vertices in which the parts are as equal in size as possible. 
Therefore, we have
\begin{equation*}
\begin{split}
\sum_{h_1\ne h_2\in \mathcal{H}} {\mathbb E}_{X\sim \mathcal{D}} [h_1(X)h_2(X)]^2 \geq 
\sum_{h_1\ne h_2\in \mathcal{H}: |{\mathbb E}_{X\sim \mathcal{D}} [h_1(X)h_2(X)]|> 1/d} \frac{1}{d^2} = \frac{{|\mathcal{H}| \choose 2}-|\mathcal{E}|}{d^2} \\ \geq 
\frac{|\mathcal{H}|(|\mathcal{H}|-1)-2T(|\mathcal{H}|, d-1)}{2d^2}.
\end{split}
\end{equation*}
Using $T(|\mathcal{H}|, d-1) \leq \frac{|\mathcal{H}|^2}{2} (1-\frac{1}{d-1})$, we conclude ${\rm BB}(\mathcal{H},\mathcal{D}) \geq \frac{|\mathcal{H}|^2}{2d^2(d-1)}-\frac{|\mathcal{H}|}{2d^2}$. We have
$$
\frac{|\mathcal{H}|^2}{2d^2(d-1)}-\frac{|\mathcal{H}|}{2d^2}\geq \frac{|\mathcal{H}|^2}{2d^3}\Leftrightarrow |\mathcal{H}|\geq d(d-1).
$$ 
Thus, if $|\mathcal{H}|>d^2$, we have ${\rm BB}(\mathcal{H},\mathcal{D})\geq \frac{|\mathcal{H}|^2}{2d^3}$. The latter implies $({\rm SQ-dim}(\mathcal{H}, \mathcal{D})+1)^3=d^3\geq \frac{|\mathcal{H}|^2}{2{\rm BB}(\mathcal{H},\mathcal{D})}$. If $|\mathcal{H}|\leq d^2$, then ${\rm SQ-dim}(\mathcal{H}, \mathcal{D})+1=d\geq |\mathcal{H}|^{1/2}$.
\end{proof}
\begin{remark} Let $\psi: {\mathbb R}\to \{-1,+1\}$ be of bounded variation. Our proof of the inequality~\eqref{main-res} was based on the intermediate result ${\rm BB}(\mathcal{H}_A, U([0,1]))\Lt A\|\psi\|_{BV}^4\log^{5}(A+1)$ proved in Section~\ref{high-freq} where $U([0,1])$ is a uniform distribution on $[0,1]$ (see the inequality~\eqref{BBbound1}). The previous theorem automatically gives us $${\rm SQ-dim}(\mathcal{H}_A, U([0,1]))\Gt \|\psi\|_{BV}^{-4/3} A^{1/3}\log^{-5/3}(A+1).$$
\end{remark}
\begin{remark} Analogously, the proof of the inequality~\eqref{discrete} for the last bit of modular multiplication was based on the intermediate result ${\rm BB}(\mathring{\mathcal{H}}_p, U({\mathbb Z}^\ast_p))\Lt pr^2(\log_2 p+1-r)^2$ proved in Section~\ref{discrete} where $U({\mathbb Z}^\ast_p)$ is a uniform distribution on ${\mathbb Z}^\ast_p$ (this can be directly derived from the inequality~\eqref{BBbound2} and Lemma~\ref{lem:sum_of_squares}). The previous theorem automatically gives us $${\rm SQ-dim}(\mathring{\mathcal{H}}_p, U({\mathbb Z}^\ast_p))\Gt \frac{p^{1/3}}{r^{2/3}(\log_2 p+1-r)^{2/3}}.$$
For bits higher than the last, definition~\ref{sq-dim} is not helpful, due to the fact that classes $\{-1,+1\}$ are unbalanced in this case.
\end{remark}
Obtained lower bounds on the statistical query dimension demonstrate that learning high-frequency functions, or learning modular arithmetic, are hard tasks not only for gradient-based algorithms but also in the general SQ model.

\section{Conclusions and open problems}
Our paper is written in the framework of the Statistical Query model and recently discovered connections of the SQ model with gradient-based learning. We consider a set of target functions (hypotheses) given in the form of a superposition of a linear function (with an integer coefficient) and a periodic function $\psi$ (on ${\mathbb R}$ or ${\mathbb Z}$). We observe a phenomenon of the concentration of the gradient of a loss function (when a target function is sampled uniformly from a hypothesis set) in the case of a 1-periodic function $\psi$ on ${\mathbb R}$ and in the case of a $p$-periodic function $\psi$ on ${\mathbb Z}$ defined as $\psi(x)=[x\bmod p]_r$ or $\psi(x)=x\bmod p$. This phenomenon is verified both mathematically and experimentally.  

We obtained that the variance in both cases is asymptotically bounded by an inverse of the square root of the cardinality of a hypothesis set. Experiments show that the variance decay even faster, i.e. as an inverse of the cardinality (see Figure~\ref{fig:one_over_p}). It is an open problem to obtain a sharper bound.

Another important open problem is to study the relationship between the verified barren plateau phenomenon in training modular multiplication and certain aspect of grokking, such as the reported existence of the ``Goldilocks'' zone in the weight space of training models and the ``LU mechanism''~\cite{liu2023omnigrok}.

\section*{Declarations}
\subsection*{Funding}
This research has been funded by Nazarbayev University under
Faculty-development competitive research grants program for 2023-2025 Grant \#20122022FD4131, PI R. Takhanov.

\bibliographystyle{sn-mathphys}  


\end{document}